\documentclass{article}

\usepackage{microtype}
\usepackage{graphicx}
\usepackage{subfigure}
\usepackage{booktabs} % for professional tables

\usepackage{hyperref}

\usepackage{amsmath}
\usepackage{amssymb}
\usepackage{mathtools}
\usepackage{amsthm}
\usepackage{color}
\usepackage{colortbl}
\usepackage{multirow}
\usepackage[capitalize,noabbrev]{cleveref}

\newcommand{\RR}{\mathbb{R}}

\newcommand{\CC}{\mathbb{C}}

\newcommand{\Sph}{\mathbb{S}}
\newcommand{\Sch}{\mathcal{S}}

\newcommand{\dd}{\mathrm{d}}
\newcommand{\M}{\mathcal{M}}
\newcommand{\num}{d}

\newcommand{\F}{\mathcal{F}}
\newcommand{\G}{\mathcal{G}}

\newcommand{\X}{\mathcal{X}}

\newcommand{\infJ}{\underline{J}}

\newcommand{\mud}{\phi} % density function for \mu
\newcommand{\nud}{\psi} % density function for \nu

\DeclareMathOperator*{\argmin}{arg\,min}

\DeclareMathOperator{\im}{image}

\DeclareMathOperator{\proj}{proj}

\newcommand{\iprod}[1]{\langle#1\rangle}
\newcommand{\eps}{\varepsilon}

\newcommand{\reffig}[1]{Figure~\ref{fig:#1}}

\newcommand{\refsec}[1]{\S~\ref{sec:#1}}
\newcommand{\refprop}[1]{Proposition~\ref{prop:#1}}

\newcommand{\refthm}[1]{Theorem~\ref{thm:#1}}

\usepackage[accepted]{icml2023}

\theoremstyle{plain}
\newtheorem{thm}{Theorem}
\newtheorem*{thm*}{Theorem}

\newtheorem{prop}{Proposition}

\theoremstyle{definition}
\newtheorem{dfn}{Definition}

\icmltitlerunning{How Powerful are Shallow Neural Networks with Bandlimited Random Weights?}

\begin{document}

\twocolumn[
\icmltitle{How Powerful are Shallow Neural Networks with Bandlimited Random Weights?}




\begin{icmlauthorlist}
\icmlauthor{Ming Li}{zjnu}
\icmlauthor{Sho Sonoda}{riken}
\icmlauthor{Feilong Cao}{jiliang}
\icmlauthor{Yu Guang Wang}{sjtu1,sjtu2}
\icmlauthor{Jiye Liang}{shanxi1}

\end{icmlauthorlist}

\icmlaffiliation{zjnu}{Key Laboratory of Intelligent Education Technology and Application of Zhejiang Province, Zhejiang Normal University, Jinhua, China}
\icmlaffiliation{riken}{Deep Learning Theory Team, RIKEN AIP, Tokyo, Japan}
\icmlaffiliation{jiliang}{College of Sciences, China Jiliang University, Hangzhou, China}
\icmlaffiliation{sjtu1}{Institute of Natural Sciences, Shanghai Jiao Tong University, Shanghai, China}
\icmlaffiliation{sjtu2}{School of Mathematical Sciences, Shanghai Jiao Tong University, Shanghai, China}
\icmlaffiliation{shanxi1}{Key Laboratory of Computational Intelligence and Chinese Information Processing of Ministry of Education, School of Computer and Information Technology, Shanxi University, Taiyuan, China}
%\icmlaffiliation{shanxi2}{School of Computer and Information Technology, Shanxi University, Taiyuan, China}
\icmlcorrespondingauthor{Ming Li}{mingli@zjnu.edu.cn}
\icmlcorrespondingauthor{Sho Sonoda}{sho.sonoda@riken.jp}

\icmlkeywords{Machine Learning, ICML}

\vskip 0.3in
]



\printAffiliationsAndNotice{}  % leave blank if no need to mention equal contribution

\begin{abstract}
We investigate the expressive power of depth-2 bandlimited random neural networks. A random net is a neural network where the hidden layer parameters are frozen with random assignment, and only the output layer parameters are trained by loss minimization. Using random weights for a hidden layer is an effective method to avoid non-convex optimization in standard gradient descent learning. It has also been adopted in recent deep learning theories. Despite the well-known fact that a neural network is a universal approximator, in this study, we mathematically show that when hidden parameters are distributed in a bounded domain, the network may not achieve zero approximation error. In particular, we derive a new nontrivial approximation error lower bound. 
The proof utilizes the technique of ridgelet analysis, a harmonic analysis method designed for neural networks. This method is inspired by fundamental principles in classical signal processing, specifically the idea that signals with limited bandwidth may not always be able to perfectly reconstruct the original signal.
We corroborate our theoretical results with various simulation studies, and generally, two main take-home messages are offered: (i) Not any distribution for selecting random weights is feasible to build a universal approximator; (ii) A suitable assignment of random weights exists but to some degree is associated with the complexity of the target function.
\end{abstract}

\section{Introduction}\label{sec_intro}
In recent years, there has been a growing interest in utilizing random methods for training neural networks as they have been shown to have the potential to significantly accelerate the training process, particularly for large-scale datasets and real-time processing requirements \cite{scardapane2017randomness, cao2018review}.
In this study, we examine the capability of shallow random networks in a \emph{bandlimited} context, where the hidden parameters are confined to a specific range. Past research has, in some cases, deliberately or unintentionally limited the distribution range of parameters. For instance, uniform distributions were used due to technical limitations, or normal random vectors with insufficiently small variance were employed due to the default configurations of the software.
Analogous to the well-established principle in signal processing (or Fourier analysis) that bandlimited signals may not be able to accurately reproduce the original signal, bandlimited neural networks may not fully exhibit their function approximation capabilities. However, unlike classical Fourier analysis, the correlation between bandwidth and approximation error has yet to be clearly defined. One challenge in this area is that neural networks do not possess an orthonormal basis, but rather a frame. Through the use of ridgelet analysis, a Fourier-like analysis developed specifically for neural networks, we have derived a new lower bound for approximation error.

This study considers a shallow neural network $g_\num(x) = \sum_{j=1}^\num c_j \sigma( a_j \cdot x - b_j )$ of input $x \in \RR^m$ with activation function $\sigma$ and parameters $(a_j, b_j, c_j) \in \RR^m \times \RR \times \RR$ for each $j \in [\num]:=\{1,\ldots,d\}$,
and the random training method with two steps:
\begin{enumerate}
    \item[] \texttt{Step~I:} Randomly initialize hidden parameters $(a_j,b_j)$ to a given data-independent probability distribution $Q(a,b)$, and freeze them; then
    \item[] \texttt{Step~II:} Statistically determine output parameters $c_j$ given a dataset $D_n = \{ (x_i,y_i) \}_{i=1}^n$.
\end{enumerate}
%\switch{The setting covers a wide range of activation functions, including radial basis functions (RBFs), the hyperbolic tangent (Tanh), and the rectified linear unit (ReLU).}{}

A variety of neural network architectures have been developed that utilize random training methods, including random vector functional-link (RVFL) networks \cite{igelnik1995stochastic}, random feature expansions \cite{rahimi2008random,rahimi2008uniform, rahimi2009weighted}, random weight networks \cite{saxe2011random}, stochastic configuration networks \cite{li2019robust,li20192,wang2017robust,wang2017stochastic}, graph convolutional networks with random weighths \cite{huang2022graph}, and certain versions of over-parametrized networks \cite{Belkin2019a,Daniely2017,Ghorbani2019,Yehudai2019}.
A kernel function defined by the inner product of feature maps: $k(x,x') = \int_{\RR^m \times \RR} \sigma(a \cdot x - b) \sigma( a \cdot x' - b ) \dd Q(a,b)$ is a special case of random training methods because we can regard this as a sum of infinitely many random samples $(a,b) \sim Q$ \cite{Bach2017a,Cho2009,Suzuki2017}. On the other hand, Bayesian neural networks \cite{Neal1996} are \emph{not} strictly based on the random training method in consideration since the distribution $Q$ is a ``posterior'', which contains the information of the dataset $D_n$; nor is the lazy learning \cite{Jacot2018} since its hidden parameters are not strictly frozen.

The random training method has the remarkable trick of ``convexification''. It frees us from inevitable non-convexity in the standard gradient descent training. The non-convexity is caused by the hidden parameters $(a_j,b_j)$ (and not by theoutput parameters $c_j$). In the random training setting, we do not optimize the parameters in Step~I, but only the output parameters $c_j$ in Step~II. %, to solve a linear system \yy = \Sph \cc + \ee\yy = \Sph \cc + \ee, where \yy, \cc\yy, \cc and \ee\ee are (mm-, \num\num- and mm-) column vectors with entries y_iy_i (output response), c_jc_j (output parameter) and e_ie_i (residual); and \Sph\Sph is the m \times pm \times p-design matrix with entries S_{ij} = \sigma(a_j \cdot x_i - b_j)S_{ij} = \sigma(a_j \cdot x_i - b_j). \red{Recent ``over-parametrization'' theories are basically stating} that if the parameter number \num\num exceeds the size nn of dataset D_nD_n, then the design matrix \Sph\Sph likely to be full row rank, which means that we can achieve the zero-training error: \ee = \mathbf{0}\ee = \mathbf{0}.
This ``randomization'' trick is beneficial both for theory and applications, {and} has recently been adopted not only in practical algorithms but also in the theoretical study of deep learning \cite{Belkin2019a,Jacot2018,louart2018random,Malach2020,pennington2017nonlinear}. While the potential of random neural networks has been explored, the expressivity of these networks has not been extensively studied. In this study, we aim to shed light on this topic by highlighting a limitation in expressivity of random nets. Our main contribution is the introduction of a new \emph{approximation lower bound} for bandlimited shallow neural networks.
% trained by random methods 
% (see the end of the section for details on the notation used therein).

% \begin{itemize}
%     \item \red{use a simple upper bound instead of f \in L^1f \in L^1}
%     \item \red{write the norms explicitly clarifying the dependence}
% \end{itemize}
\textbf{Main Theorem (simplified).}
\textit{
Let $\Omega \subset \RR^m$ be a bounded open set with smooth boundary, and put $K := \overline{\Omega}$ be its closure. Suppose $f:\Omega\to\RR$ be an $L^2$-Sobolev function on $\Omega$ with order at least $s \in (1/2,\infty]$. 
%Let $f \in L^1(\RR^m) \cap H^s(\RR^m) (1/2 \le s \le \infty)$ be a function %that is supported on a compact domain $K \subset \RR^m$.
Let $\lambda > 0$, and put $V := \{ (a,b) \in \RR^m \times \RR \mid |a| \le \lambda, |b| \le \lambda \}$.
%$V_a = \{ a \in \RR^m \mid |a| \le \lambda \}$ and $V_b \subset \RR$ 
%be the domain of hidden parameters $a$ and $b$.
Suppose $\sigma:\RR\to\RR$ is bounded, self-admissible, and 
% \in L^\infty(\RR) \in L^\infty(\RR) is self-admissible, and that 
the constant $C_{\sigma,P} := \sup_{(a,b) \in V}\| \sigma_{a,b} \|_{L^2(K)}$ exists finite.
Consider approximating $f$ with a bandlimited shallow neural network $g_\num(x) = \sum_{j=1}^\num c_j \sigma( a_j \cdot x - b_j )$.
%Namely, draw i.i.d. samples (a_j,b_j) \sim U( V )(a_j,b_j) \sim U( V ) by \texttt{Step~I}; and determine c_jc_j by least squares by \texttt{Step~II}.
Then, the approximation error is lower bounded as
\begin{eqnarray*}
 &&\| f - g_\num \|_{L^2(K)}^2
    %&\ge \| f \|_{L^2(\RR^m)}^2 - \| f \|_{L^1(\RR^m)}^2 V_m (\lambda \wedge \vartheta)^m \kappa \nonumber \\
    %&\qquad - \| f \|_{W^{s,1}(\RR^m)}^2 \Omega_m \frac{\vartheta^{-(2s-m)} - (\lambda \vee \vartheta)^{-(2s-m)}}{2s-m} \kappa,
    \ge \| S^*[f] \|_{L^2(V^c)}^2
    \ge \| f \|_{L^2(K)}^2 -\nonumber\\
    &&\!\!\!\!\!\!\!\!\!C_b^2 \begin{cases}
    \| f \|_{L^1(K)}^2 \| \sigma \|_{L^\infty(\RR)}^2 \lambda^m, & \!\!\!\! \lambda \in (0,\vartheta) \\
    \| f \|_{H^s(\Omega)}^2 C_{\sigma,s}^2 \!\left(\!\frac{-1}{2s} \lambda^{-2s}\! +\! \frac{2s+m}{2sm} \vartheta^{-2s} \right), & \!\!\!\! \lambda \in [\vartheta, \infty)
    \end{cases}.
\end{eqnarray*}\label{eq:lb.main.rough}
Here, $S^*$ is the adjoint of integral representation operator $S$;
and the constants $C_b, C_{\sigma,s}$ and threshold $\vartheta >0$ depend on norms of $f$, dimension $m$ and smoothness $s$. (See \reffig{lbound} for the outline, and Section \ref{sec3} for more detailed statement.)}
\begin{figure}[h]
    \vspace{-1ex}
    \centering
    \includegraphics[width=0.8\linewidth]{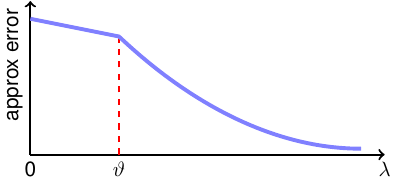}
    \vspace{-2ex}
    \caption{Outline of the approximation lower bound.}
    \label{fig:lbound}
    \vspace{-2ex}
\end{figure}

% \textbf{Remarks.}
Our results, as depicted in \reffig{lbound}, demonstrate that the lower bound:
(i) holds true regardless of whether the hidden parameters are random or deterministic,
(ii) is always non-negative,
(iii) converges to $0$ as bandwidth $\lambda$ tends to $\infty$,
which agrees with the universality of the network,
(iv) is continuous at $\lambda = \vartheta$,
and (v) depends on the smoothness $s$ of target function $f$ when $\lambda \ge \vartheta$. 
This implies that the lower bound is more significant when the target function is smoother.
Thus, if the domain of hidden parameters is \emph{bandlimited}, the approximation error may not reach $0$. To better illustrate this interesting phenomenon, we provide quantitative demonstration based on two simple examples in Appendix \ref{Quantitative}. 

\textbf{Additional remarks to avoid potential confusions.}
In the \textbf{Main Theorem}, \emph{randomness is not essential/required but bandlimiting is.} Nevertheless, we focus on random nets because we can find a lot of usecases both in theorertical and practical aspects. As long as the hidden unit number is finite, random nets are always bandlimited. We may list two more examples:
(1) Random nets with fully-supported proposal distribution $Q$ (such as Gaussian), because hidden unit parameters are in a bounded domain with high probability.
(2) NNs trained by gradient descent in the lazy learning regime, because the lazy assumption (that final parameters are close to the initial parameters) means that the final parameters are distributed in a bounded domain (with high probability).

%\textbf{Demo of Parameter Distribution and Ridgelet Spectrum:}
%\textbf{Proof Idea.}
% The key instruments to derive the lower bound are the \emph{integral representation} of neural networks and \emph{ridgelet analysis} \cite{Barron1993,Candes.HA,murata1996integral,Sonoda2015}, which are developed as the harmonic analysis for neural networks.
% Recall the well-known fact in Fourier analysis: if the Fourier spectrum ˆf(ω)\widehat{f}(\omega) of a function f∈L2(\RRm)f \in L^2(\RR^m) is compactly supported, then ff is very smooth (i.e., holomorphic). In other words, if the domain of the (angular) frequency ω\omega is \emph{bandlimited}, then the expressive power is limited. %(\red{Ming: Again, the remark is not clear, because one needs to see how the lower bound behaves for \lambda=\infty\lambda=\infty.})
% %in the Fourier series expansion f(x) = \sum_{j=-\infty}^\infty c_j \exp( i x \cdot \omega_j )f(x) = \sum_{j=-\infty}^\infty c_j \exp( i x \cdot \omega_j ), if the frequency parameter \omega_j\omega_j is bandlimited, then the
% In ridgelet analysis, the hidden parameter (a,b)(a,b) plays a similar role as frequency ω\omega, and so we can show that if the ``frequency'' (a,b)(a,b) is bandlimited, then the expressive power of the network g\numg_\num is limited. %\red{Again, the remark is not clear, because one needs to see how the lower bound behaves for \lambda=\infty\lambda=\infty.}

\textbf{Notation.}
For any complex number $z$, we denote by $\overline{z}$ the complex conjugate of $z$.
For any subset $A \subset X$ of %an entire
a set $X$, $A^c (= X \setminus A)$ denotes the complement of $A$. For an activation function $\sigma$, we write $\sigma_{a,b}(x) := \sigma(a\cdot x - b)$. For any vector $x \in \RR^d$, we denote $\iprod{x} := (1+|x|^2)^{1/2}$, where $|x|$ is the Euclidean norm of $x$.
$\Sch(\RR)$ denotes the space of rapidly decreasing smooth functions, or the Schwartz test functions, on $\RR$; and $\Sch'(\RR)$ denotes the space of tempered distributions, or the topological dual space of $\Sch(\RR)$.
$H^s(\Omega)$ denotes the $L^2$-Sobolev space on open set $\Omega$ of order $s (\in \RR)$.
In order to avoid confusion, we use $\rho^\sharp(\omega) := \int_\RR \rho(t)\exp(-it\omega)\dd t$ for $1$-dimensional Fourier transform of $\rho \in L^2(\RR)$, and $\widehat{f}(\xi) := \int_{\RR^m} f(x) \exp(-i \xi \cdot x) \dd x$ for $m$-dimensional Fourier transform of $f \in L^2(\RR^m)$. By the terms `random neural networks', `random nets', or `neural nets with random weights' we mean the same thing.

\section{Integral Representation and Ridgelet Transform}\label{sec2}

In this section, we introduce a few basics of the integral representation theory and ridgelet analysis, then provide several important (known) propositions that will be used in proving our main results. In Appendix~\ref{sec:int_rep_theory}, we further supplemented the backgrounds and detailed (but quick) overview.

\subsection{Integral Representation of Neural Nets}
The integral representation is a handy tool for the analysis of neural networks with a variable number of hidden units.

Let $V$ be a Borel subset in $\RR^m \times \RR$, and 
$\M(V)$ be the space of all signed Radon measures on $V$.
We set $V$ to be a space of hidden parameters $(a,b)$, and call an element $\mu \in \M(V)$ a \emph{parameter distribution}.

Let $\sigma : \RR \to \RR$ be an \emph{activation function}. We always assume that the activation function $\sigma$ is associated with a function $\rho$ that satisfies the admissibility condition,
%is a tempered distribution (\Sch'\Sch') on \RR\RR but not a polynomial function, and (A2) that it
%which is defined later,
which is a sufficient condition for the neural network to have the universal approximation property (see \refsec{ac} and \refprop{reconst}).
%\switch{Many typical activation functions can satisfy the condition, for example, such as radial basis functions (RBFs), the hyperbolic tangent (Tanh), and the rectified linear unit (ReLU).}{}
%We remark that by carefully modifying the Plancherel theorem, our main result can also cover the rectified linear unit (ReLU). However, we exclude this for the sake of simplicity.

\begin{dfn}
The \emph{integral representation} of a neural network is defined as an integral transform of the parameter distribution $\mu \in \M$:
\begin{align}
    S[\mu](x) := \int_{\RR^m \times \RR} \sigma(a \cdot x - b) \dd \mu(a,b), \quad x \in \RR^m.
\end{align}
\end{dfn}
We have two motivations to employ the integral representation introduced above.
First, it provides a unified expression for a variable number $\num$ of parameters including an infinite number of parameters. As the integral suggests, it is \emph{formally} an infinite version of the ordinary finite neural network. Namely, whereas the finite net  $g_\num = \sum_{j=1}^\num c_j \sigma( a_j \cdot x - b_j )$ is a weighted sum of a finite number of hidden units $\sigma(a_j \cdot x - b_j)$ with weights $c_j$ and indices $j \in [\num]$, the infinite net $S[\mu]$ is a weighted integral of an infinite number of hidden units $\sigma(a \cdot x - b)$ with weight function $\mu(a,b)$ and ``indices'' $(a,b) \in V$. Nevertheless,
% \textbf{Three remarks.}
%As remarks,
we can also \emph{express a finite net} as $g_\num = S[\mu_\num]$ by letting $\mu_\num = \sum_{j=1}^\num c_j \delta_{(a_j,b_j)}$ with Dirac measures $\delta_{(a_j,b_j)}$, because we assume that a parameter distribution $\mu$ is a Radon measure, which includes both continuous densities and singular masses. In other words, the integral representation is \emph{not} a counterpart of the finite models, but it is an extension of the finite models.
Second, the map $S$ is linear in $\mu$. Since the non-linear parameters {$(a,b)$} are ``integrated out'' in the integral representation (as ``marginalized out'' in the Bayesian literature), we are now free from the non-linearity of neural networks. 
%Finally, in order to cover a wide range of activation functions, we will use a slightly extended definition of S[\mu]S[\mu] in the proof sections. (But we will not use this extended version in the main sections for the sake of simplicity).

\subsection{Admissibility Condition}\label{sec:ac}
\begin{dfn}%[Admissibility condition]
Given an activation function $\sigma : \RR \to \CC$, we say that a function $\rho : \RR \to \CC$ is \emph{admissible} when it satisfies the \emph{admissibility condition}
\begin{align}
(2 \pi)^{m-1} \int_\RR \sigma^\sharp(\omega)\overline{\rho^\sharp(\omega)}|\omega|^{-m} \dd \omega = 1.
\end{align}
\end{dfn}
This condition seems technical, but is typical in the context of wavelet transforms (see e.g., \citet{Mallat2009}). It simply requires the $|\omega|^{-m}$-weighted inner product of $\sigma^\sharp$ and $\rho^\sharp$ to be finite (not zero nor infinite). Therefore, this is not a strong condition and we can find, in general, an infinite number of different $\rho$'s for a given $\sigma$. For example, if $\sigma$ is a gaussian, then its Fourier transform $\sigma^\sharp$ is again another gaussian, and we can find a ``family of'' particular solutions: $\rho^\sharp(\omega) = C |\omega|^m \phi^\sharp(\omega)$ with an arbitrary Schwartz function $\phi \in \Sch(\RR)$ (as long as the integral is finite and not zero) and an appropriate normalizing constant $C>0$. %\cite{Sonoda2015acha}.
We refer to \citep[\S~6.2]{Sonoda2015acha} for more examples. Finally, when $\rho=\sigma$, we say that $\rho$ (or $\sigma$) is \emph{admissible with itself}, or \emph{self-admissible}.

\subsection{Ridgelet Transform}
The ridgelet transform $R$ is, in a nutshell, a right inverse operator to the integral representation operator $S$.
Given a function $f \in L^2(\RR^m)$, consider finding an unknown parameter distribution $\mu \in \M$ that satisfies the integral equation $S[\mu] = f$.
As we would describe later, the solution to this integral equation is not unique, and the ridgelet transform provides a particular solution to this equation.

%Given an activation function \sigma : \RR \to \CC\sigma : \RR \to \CC, we say that a function \rho : \RR \to \CC\rho : \RR \to \CC is \emph{admissible} when it satisfies the \emph{admissibility condition}
%\begin{align}
%(2 \pi)^{m-1} \int_\RR \sigma^\sharp(\omega)\overline{\rho^\sharp(\omega)}|\omega|^{-m} \dd \omega = 1.
%\end{align}
%This condition seems technical, but typical in the context of wavelet transform. It simply requires that the |\omega|^{-m}|\omega|^{-m}-weighted inner product of {\sigma^\sharp\sigma^\sharp} and {\rho^\sharp\rho^\sharp} to be finite (not zero nor infinite). Therefore, this is not a strong condition and we can find, in general, an infinite number of different \rho\rho's for a given \sigma\sigma. For example, if \sigma\sigma is a gaussian, then its Fourier transform \sigma^\sharp\sigma^\sharp is again another gaussian, and we can find a ``family of'' particular solutions: \rho^\sharp(\omega) = C |\omega|^m \phi^\sharp(\omega)\rho^\sharp(\omega) = C |\omega|^m \phi^\sharp(\omega) with an arbitrary Schwartz function \phi \in \Sch(\RR)\phi \in \Sch(\RR) (as long as the integral is finite and not zero) and an appropriate normalizing constant CC. %\cite{Sonoda2015acha}.
%We refer to \cite[\S~6.2]{Sonoda2015acha} for more examples.

\begin{dfn}
For every $f \in L^p(\RR^m) (p=1,2)$, the \emph{ridgelet transform} of $f$ with respect to $\rho \in \Sch(\RR)$ is given by
\begin{equation*}
    R[f](a,b) := \int_{\RR^m} f(x) \overline{\rho(a \cdot x - b)} \dd x, \quad (a,b) \in \RR^m \times \RR.
\end{equation*}
\end{dfn}
We provide two important propositions as a basic preparation for the main theoretical analysis performed in Section 3. Their proofs are reported in Appendix \ref{sec:int_rep_theory}.
%\switch{It should be noted, however, that two more theorems that extend these propositions are provided with detailed proofs in the supplementary material, by which the case of some non-integrable activation functions are covered.}{}
%\switch{We remark that some non-integrable activation functions such as the Tanh and ReLU cannot meet self-admissible in the present setup. In the proof sections elaborated in the supplementary material, we extend the setup to include these activation functions. However, for the sake of simplicity, we do not use the extended setup in the main sections.}{}
%(See \cite[Theorem~5.6, 5.9]{Sonoda2015acha} for the original proofs of the propositions.)

\begin{prop}[Reconstruction formula] \label{prop:reconst}
Let $f \in L^p(\RR^m) (p=1,2)$.
Provided that $\rho \in \Sch(\RR)$ is admissible with an activation function $\sigma \in \Sch'(\RR)$, then we have
\begin{eqnarray*}
   S[ R[f] ](x) &=& \int_{\RR^m \times \RR} R[f](a,b) \sigma( a \cdot x - b  ) \dd a \dd b\\&= &f(x), \quad x \in \RR^m.
\end{eqnarray*}
%Here, \red{the equality is in the senses of pointwise at every continuous point of ff, in L^1L^1 and in L^2L^2.}
\end{prop}
We have two interpretations for Proposition~\ref{prop:reconst}.
First, recall that $S[\mu]$ represents a neural network. Then,
the reconstruction formula implies the \emph{universal approximation property}, %in the sense of integral representation
because a neural network $S[\mu]$ can express any function $f \in L^p(\RR^m) (p=1,2)$ by letting $\mu = R[f]$.
Second, recall the Fourier inversion formula: $F^{-1}[\widehat{f}](x) = (2 \pi)^{-m} \int_{\RR^m} \widehat{f}(\xi) \exp(i \xi \cdot x) \dd \xi = f(x)$. Then, we can find a clear correspondence of $S$ with $F^{-1}$, $R[f]$ with $\widehat{f}$, and $\sigma(a\cdot x - b)$ with $\exp(i \xi \cdot x)$.
However, we should also remark the difference that by the non-uniqueness of admissible functions $\rho$, the ridgelet transform $R[f]$ is not unique either. This means that $R$ is \emph{not} the strict inverse operator to $S$, but only a right inverse operator to $S$.

\begin{prop}[Plancherel theorem] \label{prop:plancherel}
Let $f \in L^2(\RR^m)$.
Provided that $\sigma$ is self-admissible, %namely, it is admissible with itself ($\rho = \sigma$),
one has
%Then,
%\begin{align}
$\| R[f] \|_{L^2(\RR^m \times \RR)} = \| f \|_{L^2(\RR^m)}.$
%\end{align}
\end{prop}
The Plancherel theorem plays a key role in our main results. As to be displayed in \reffig{spectrum}, a ridgelet spectrum $R[f]$ has a long tail. If the ridgelet spectrum $R[f]$ is truncated, then the Plancherel theorem implies that we cannot reconstruct $f$ without loss.

\subsection{Proof Idea Behind Main Theorem}
%To give an intuition of our main theorem, 
%As clarified in Section~\ref{sec2}, 
To enhance the readers’ understanding of our main theorem, 
we present a visual example of \emph{how real parameters are distributed}. Without this visualization, some readers may imagine/assume that neural network parameters are distributed randomly, with relatively simple structures such as uniform distribution and normal distribution. On the contrary, \emph{the parameter distribution has the structure of ridgelet transform.} With this illustration, we can intuitively/visually understand how/why the expressive power is lost when the parameter distribution is truncated to a bounded domain.

%The ridgelet transform can compute how the parameters are distributed for a neural network to approximate given function $f$. 
\reffig{spectrum}, produced by \citet{Sonoda2018a}, visualizes parameter distributions in two ways. %These figures are presented initially in \cite{Sonoda2018a}.

\begin{figure}[h]
\centering
\subfigure[GD trained parameters]{\includegraphics[width=0.238\textwidth]{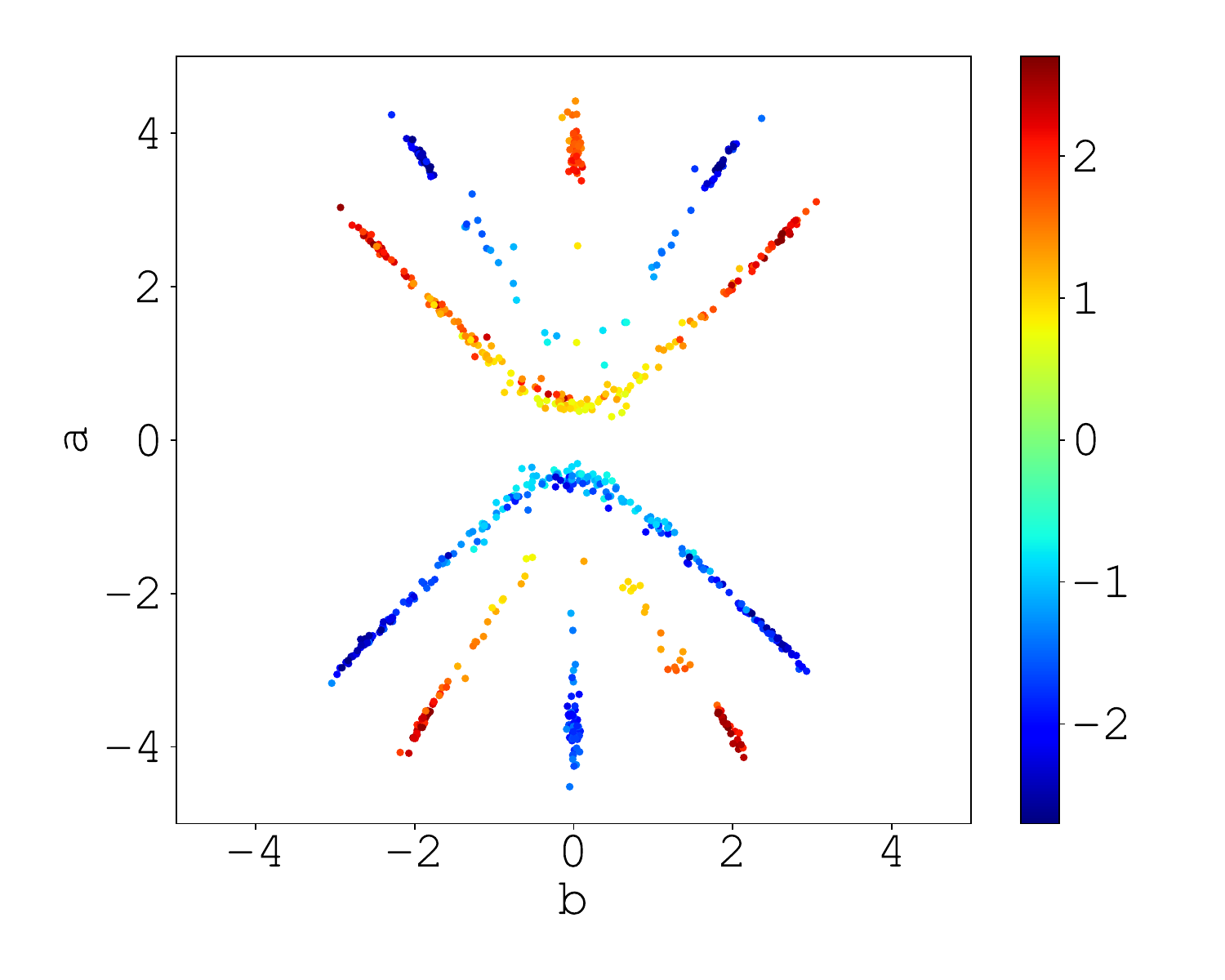}}
\subfigure[Ridgelet spectrum]{\includegraphics[width=0.238\textwidth]{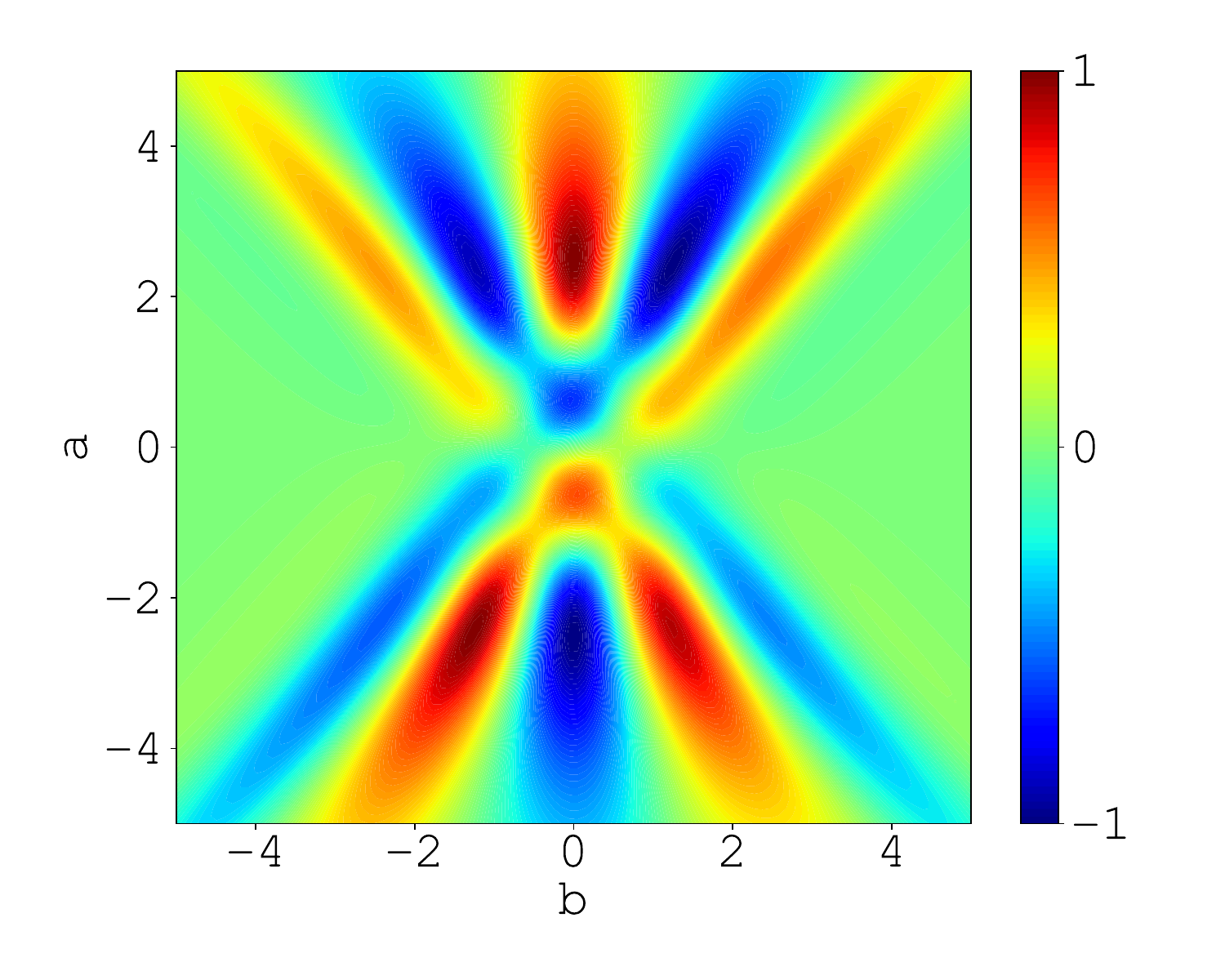}}
\vspace{-2ex}
\caption{Example of parameter distributions}\label{fig:spectrum}
\end{figure}
\vspace{-1ex}

Both figures are obtained from the common dataset $D_n = \{ ( x_i, y_i )\}_{i=1}^n$ that is generated by the function $y = f(x) = \sin 2 \pi x$. \reffig{spectrum}(a) shows the distribution of the parameters $(a_j,b_j,c_j)$, which are obtained from many neural networks trained on the common dataset $D_n$ by gradient descent (GD); and \reffig{spectrum}(b) shows the ridgelet spectrum $R[f](a,b)$ approximated by numerical integration evaluated at each point $(a,b)$.

Despite the fact that two figures are obtained from different procedures, gradient descent and numerical integration, they have an apparent intriguing resemblance. The shapes of the distributions are $10$-point star shaped. In other words, the trained parameters $(a_j, b_j, c_j)$ concentrate on high intensity areas in the ridgelet spectrum. %\refthm{minimizer} provides an extended version of the theoretical derivation and explanation of this phenomenon. 
This phenomenon that GD converges to the ridgelet spectrum is initially reported by \citet{Sonoda2018a}, and recently given a mathematical justification by \citet{Sonoda2021aistats}.

Based on the visualized results, one can naturally \emph{conjecture} that if the parameter space is bandlimited, that is, the spectrum is truncated to a compact domain such as $|a| \le \lambda$ and $|b| \le \kappa$, then the neural network loses the universal approximation property. In other words, there exists a class of functions that a bandlimited network cannot reconstruct. Overall, that is the primary idea behind our main theorem, i.e., we quantify and prove this conjecture by carefully estimating the tail of the ridgelet spectrum.

%\paragraph{Remark 1.}

\section{Main Results}\label{sec3}

%For the sake of readability, we write $\| \cdot \| = \| \cdot \|_{L^2(\RR^m)}$.
%Given a function $f \in L^2(K)$, 
For theoretical analysis, we reformulate the random training method at the beginning of the introduction as follows: Let $\Omega$ be a bounded open set with smooth boundary, put $K := \overline{\Omega}$, and let $V$ be a Borel set in $\RR^m\times\RR$.
\begin{itemize}
%    \item[] \texttt{Step~I':} Draw $\{ (a_j,b_j) \}_{j=1}^\num \sim \uni(V)$, 
    \item[] \texttt{Step~I':} Let $\{ (a_j,b_j) \}_{j=1}^\num$ be arbitrary $\num$ points in $V$, 
    and let $\M(\num) := \{ \sum_{j=1}^\num c_j \delta_{(a_j,b_j)} \mid c_j \in \RR \}$,
    \item[] \texttt{Step~II':} Let $\mu^\circ_\num := \argmin_{\mu \in \M(\num)} \| f - S[\mu] \|_{L^2(K)}^2$, and let $g^\circ_\num := S[\mu^\circ_\num]$.
\end{itemize}
Here, the generation process of $(a_j,b_j)$ need not be random as long as these are inside of $V$.

The main goal of this section is to lower bound the approximation error $\| f - g^\circ_{\num}\|_{L^2(K)}$. This is also a lower bound on $\| f - g_{\num}\|_{L^2(K)}$, where $g_d$ is provided by Step~II in Section \ref{sec_intro}, since $\| f - g^\circ_{\num}\|_{L^2(K)} \leq \| f - g_{\num}\|_{L^2(K)}$ by construction. Unlike the Fourier or Taylor series expansions, the rate of approximation lower bound for a finite $\num$ is unknown, and it is known as a (complicated) open question (see \citet{kainen.survey} for more details). To circumvent this difficulty, we only estimate the approximation error achieved by infinite 
minimizers, $\mu \in \M(V)$, 
which exists as a consequence of the extreme value theorem,
and lower bounds the approximation error achieved by its finite minimizers as follows:
\begin{eqnarray}
 \| f - g^\circ_\num \|_{L^2(K)}^2
    &=& \min_{\mu \in \M(\num)} \| f - S[\mu] \|_{L^2(K)}^2  \nonumber \\&\ge& \inf_{\mu \in \M(V)} \| f - S[\mu] \|_{L^2(K)}^2, \label{eq:lb.S}
\end{eqnarray}
where the inequality above is an immediate consequence of the inclusion $\M(\num) \subset \M(V)$.
Namely, just contrary to the case of estimating upper bounds, the lower bound for more expressive (=infinite) networks is automatically valid for less expressive (=finite) networks.
%So, our result covers finite networks.

%In \refthm{minimizer}, we show that the minimizer can be represented by the ridgelet spectrum.
In \refthm{minimizer}, we show that the infinite minimum is lower bounded by the tail part of the ridgelet spectrum.
\begin{thm} \label{thm:minimizer}
Let $f \in L^2(K)$ be a square-integrable function that is supported in the compact domain $K$. Assume that $\sigma$ is self-admissible, and that the constant $C_{\sigma,P} := \sup_{(a,b) \in V}\| \sigma_{a,b} \|_{L^2(K)}$ exists finite.
Then, the approximation error 
%$\inf_{\mu \in \M(V)} \| f - S[\mu] \|_{L^2(K)}^2$ 
is lower bounded 
%by the volume of the ridgelet spectrum outside of the parameter domain $V$ 
as follows:
\begin{align}
    &\inf_{\mu \in \M(V)} \| f - S[\mu] \|_{L^2(K)}^2 \notag\\
    &\ge \| S^*[f] \|_{L^2(V^c)}^2 = \int_{V^c} |S^*[f](a,b)|^2 \dd a \dd b, \label{eq:lb.R}
\end{align}
where $S^*$ is the adjoint operator of $S$.
\end{thm}
%\noindent %See \refsupp{proof.minimizer} for the proof.
% Combine this representation theorem with the reconstruction formula $f = S[R[f]]$ and the Plancherel identity $\| f \| = \| R[f] \|$, we can obtain a key equation to estimate the lower bound of the approximation error:
% \begin{align}
%     \inf_{\mu \in \M(V)} \| f - S[\mu] \|^2
%     = \| f - S[\mu_V] \|^2
%     = \int_{V^c} |R[f](a,b)|^2 \dd a \dd b. \label{eq:lb.R}
% \end{align}
Namely, if the tail part $S^*[f]|_{V^c}$, or the ridgelet spectrum outside the parameter domain $V$, does not vanish, then the tail bound $\| S^*[f] \|_{L^2(V^c)}^2$ inevitably lower bounds the approximation error $\| f - {g^\circ_\num} \|_{L^2(K)}$.

In order to quantify (or estimate from below) the tail bound $\| S^*[f] \|_{L^2(V^c)}$,
we exploit the property $S^*[f]=R[f]$ (valid in case of a self-admissible function $\sigma$, see Appendix \ref{sec:int_rep_theory}), and rewrite the right-hand side of \eqref{eq:lb.R} as
\begin{align*}
\| S^*[f] \|_{L^2(V^c)}^2 =\| R[f] \|_{L^2(V^c)}^2 = \|f\|_{L^2(K)}^2 - \| R[f] \|_{L^2(V)}^2.
\end{align*}
Then, the estimation problem of $\| S^*[f] \|_{L^2(V^c)}$ from \emph{below} is now turned to the estimation problem of $\|R[f]\|_{L^2(V)}$ from \emph{above}.
Thus, we can estimate the tail bound through the decay property of ridgelet spectrum, which is given by the following theorem.
%find a (tractable) dominating function ϕ(a,b)\phi(a,b) of the ridgelet spectrum such as |R[f]|≤ϕ|R[f]| \le \phi on VV.
%In \refthm{ubound}, we derive such a ϕ\phi by investigating the decay property of the ridgelet spectrum.
\begin{thm} \label{thm:ubound}
Let $f \in H^s(\Omega)$ be an $L^2$-Sobolev function on $\Omega$ with order $s \in (1/2,\infty]$. %and ρ=ϕ(s)\rho = \phi^{(s)} for some ϕ∈L∞∩Cs(\RR)\phi \in L^\infty \cap C^s(\RR).
Assume that $\rho \in L^\infty(\RR)$ be self-admissible. For $(r,u,b) \in \RR_+ \times \Sph^{m-1} \times \RR$, put 
\begin{align*}
    &\phi_a(ru) \!:= \!\!\min\!\left\{\!\| f \|_{L^1(K)} \| \rho \|_{L^\infty(\RR)}, C_{\rho,s} \Phi_s[f](u) r^{\frac{-2s-m}{2}} \!\right\},\\
    &\phi_b(ru,b) := |R[f;\rho](ru,b)| / \phi_a(ru) \quad (\le 1),
\end{align*}
where
$C_{\rho,s}^2 := \frac{2}{(2 \pi)^2} \int_{\RR} \iprod{\omega}^{-2s+1} |\rho^\sharp(\omega)|^2 |\omega|^{-m} \dd \omega$, which always exists finite; and
$\Phi_s[f]{(u)}^2 := \int_0^\infty \iprod{\omega}^{2s} |\widehat{f}(\omega u)|^2 \omega^{m-1} \dd \omega$, which satisfies $\int_{\Sph^{m-1}} \Phi_s[f]{(u)}^2 \dd u = \| f \|_{H^s}^2$.
Then, the ridgelet spectrum is upper bounded by
\begin{equation*}
    |R[f;\rho](ru,b)| \le \phi_a(ru), \quad (r,u,b) \in \RR_+ \times \Sph^{m-1} \times \RR.
\end{equation*}
%(\red{Ming: Why is an upper bound obtained here? Given the context, a lower bound would look to be more useful here.})
%Namely, the upper bound $\phi_a$ is independent of $b$, which yields that $\phi_b$ is constant in $ru$ 
%(\red{Ming: Actually it is not clear why 
%$\phi_b$ is constant in $ru$ (this looks important to get the next equality)}) and upper bounded by $1$. 
Furthermore, when $V$ is given by a product $V_a \times V_b$ with some $V_a \subset \RR^m$ and $V_b \subset \RR$, we can decompose the integral:
\begin{equation*}
    \| R[f;\rho] \|_{L^2(V_a \times V_b)} = \| \phi_a \|_{L^2(V_a)}\| \phi_b \|_{L^2(V_b)}.
\end{equation*}
\end{thm}
%See \refsupp{proof.ubound} for the proof.
Finally, by integrating the dominating functions $\phi_a$ and $\phi_b$, 
%(\red{Ming: Why ``dominating functions''? Only $\phi_a$ looks to be dominating})
we obtain an %exact
estimate of the tail bound $\| R[f] \|_{L^2(V^c)}$, as follows.
%Finally, in \refthm{mainclaim} we derive the lower bound by estimating the tail bound as
%\begin{align}
%‖\| R[f]|_{V^c} \|^2 \ge \|f\|^2 - \| \phi|_V \|^2,
%\end{align}
%which concludes the final assertion: \| f - g_\num \|^2 \ge \| f \|^2 - \| \phi|_V \|^2\| f - g_\num \|^2 \ge \| f \|^2 - \| \phi|_V \|^2.
\begin{thm}[Main Theorem] \label{thm:mainclaim}
    Let $f \in H^s(\Omega)$ with $s \in (1/2,\infty]$.
    Let $\lambda>0$, $V_a := \{ a \in \RR^m \mid |a| \le \lambda\}$ be a ball, $V_b \subset \RR$ be an arbitrary Borel set, and put $V=V_a \times V_b$. %\{ (a,b) \in \RR^m \times \RR \mid |a| \le \lambda, |b| \le \kappa/2 \}\{ (a,b) \in \RR^m \times \RR \mid |a| \le \lambda, |b| \le \kappa/2 \}.
    Assume that $\sigma \in L^\infty(\RR)$ is self-admissible, and that the constant $C_{\sigma,P} := \sup_{(a,b) \in V_a \times V_b}\| \sigma_{a,b} \|_{L^2(K)}$ exists finite.
    Put $\vartheta := (m V_m \| f \|_{L^1(K)}^2 \| \rho \|_{L^\infty(\RR)}^2 / \|f \|_{H^s(\Omega)}^2 C_{\rho,s}^2)^{-1/(2s+m)}$, where $V_m := \pi^{m/2}/\Gamma(m/2+1)$ is the volume of the $m$-unit ball.
    %Let $f$ and $\rho$ be the same as in \refcor{ubound}.
    %On $f$ and $\rho$, we impose the same assumptions in \refthm{ubound}.
    %Let C_0 = \| f \|_{L^1(\RR^m)} \| \rho \|_{L^\infty(\RR)}C_0 = \| f \|_{L^1(\RR^m)} \| \rho \|_{L^\infty(\RR)}, let C_\infty =\|f \|_{W^{s,1}(\RR^m)} \| \phi \|_{L^\infty(\RR)}C_\infty =\|f \|_{W^{s,1}(\RR^m)} \| \phi \|_{L^\infty(\RR)}, and let \vartheta := (C_0/C_\infty)^s\vartheta := (C_0/C_\infty)^s.
    %Consider approximating ff with a random net g_\num = \sum_{j=1}^\num c_j \sigma( a_j \cdot x - b_j )g_\num = \sum_{j=1}^\num c_j \sigma( a_j \cdot x - b_j ). %Namely, draw i.i.d. samples (a_j,b_j) \sim U( V_a \times V_b )(a_j,b_j) \sim U( V_a \times V_b ) by \texttt{Step~I}; and determine c_jc_j by least squares by {\texttt{Step~II}}.
%Then, the approximation error is lower bounded as
Then, for a bandlimited network $g_\num = \sum_{j=1}^\num c_j \sigma( a_j \cdot x - b_j )$ obtained by Steps~I' and II', we have the following approximation lower bounds:
%Furthermore, assume (\texttt{Step~I'}) that \{ (a_j,b_j) \}_{j=1}^\num \sim U( V )\{ (a_j,b_j) \}_{j=1}^\num \sim U( V ) be i.i.d. samples and let \M(\num) := \{ \sum_{j=1}^\num c_j \delta_{(a_j,b_j)} \mid c_j \in \RR \}\M(\num) := \{ \sum_{j=1}^\num c_j \delta_{(a_j,b_j)} \mid c_j \in \RR \}; and (\texttt{Step~II'}) that \mu_\num := \argmin_{\mu \in \M(\num)} \| f - S[\mu_\num]\|\mu_\num := \argmin_{\mu \in \M(\num)} \| f - S[\mu_\num]\| and let g_\num := S[\mu_\num]g_\num := S[\mu_\num].
\begin{eqnarray*}
&&\| f - g_\num \|_{L^2(K)}^2  \ge \inf_{\mu \in \M(V)} \| f - S[\mu] \|_{L^2(K)}^2 \nonumber \\
&&\ge \| S^*[f] \|_{L^2(V^c)}^2 = \| f \|_{L^2(K)}^2 -\| \phi_b \|_{L^2(V_b)}^2\cdot \nonumber\\
    &&\!\!\!\!\!\!\!\!\!\!\!\!\!\!\begin{cases}
    \!\| f \|_{L^1(K)}^2 \| \sigma \|_{L^\infty(\RR)} \lambda^m, &\!\!\!\! \lambda \in (0,\vartheta), \\
    \!\| f \|_{H^s(\Omega)}^2 C_{\sigma,s}^2 \!\left( \frac{-\lambda^{-2s}}{2s} \!+\! \frac{2s+m}{2sm} \vartheta^{-2s}\! \right), & \!\!\!\! \lambda \in [\vartheta,\infty).
    \end{cases} \label{eq:lb.main.rate}
\end{eqnarray*}
Here, the final bound is continuous at $\lambda = \vartheta$, always non-negative, and tends to $0$ as $\lambda \to \infty$.
\end{thm}
We provide the proofs of all the theoretical results above in Appendix~\ref{supp:proofs}.
%We note that the constant \vartheta can be modified to be independent of f, by taking \vartheta' := \sup_f \vartheta, which exists thanks to the continuity of norms: \| \cdot \|_{L^1(\RR^m)} \lesssim \| \cdot \|_{W^{k,1}(\RR^m)}\| \cdot \|_{L^1(\RR^m)} \lesssim \| \cdot \|_{W^{k,1}(\RR^m)}.
%See \refsupp{proof.mainclaim} for the proof.

\subsection{Technical Remarks}
%\paragraph{Remark 1.} %~There is no randomness anywhere in the main theorem.%
\paragraph{(Un)necessity of Randomness.}
Even though our subject is random nets, we \emph{do not need any randomness in the main theorem} because the key inequality (\ref{eq:lb.S}) holds for any realization of $\mu \in \M(V)$ (besides that the function $f$ is fixed). According to Steps~I' and II', the LHS of (\ref{eq:lb.S}) is a random variable. However, the RHS is not a random variable but a constant because it is by definition smaller than any loss-value  $J(\mu) := \| f - S[\mu] \|^2_{L^2(K)}$ for $\mu \in \M(V)$. (\ref{eq:lb.R}) ($=$ RHS of (\ref{eq:lb.S})) indicates that the lower bound on $\infJ := \inf_{\mu {\in \M(V)}} J(\mu)$ is strictly positive when the ridgelet spectrum $R[f]$ is supported on a set containing the parameter domain $V$. Thus, if the proposal distribution $Q(a,b)$ (in Step~I) is supported on a compact set $V$ and $R[f]$ has support containing $V$, then inevitably $\infJ > 0$. We may consider extensions to a fully supported distribution such as the normal distribution $N$. For this case, in contrast, to extend our main result, \emph{we need some high probability condition} that initial parameters $\{(a_j,b_j)\}_{j=1}^\num$ concentrate on a certain compact domain $V$.

%\paragraph{Remark 2.} 
\paragraph{Extension to Unbounded Activation Functions such as ReLU.}
%We note that 
It would be possible, 
%to extend the Theorem \ref{thm:mainclaim} to the cases with unbounded activation functions (such as ReLU), 
but not immediate.
The Plancherel formula (\refprop{plancherel}) is a key step to obtain the lower bound in Theorem \ref{thm:minimizer}, and 
the self-admissible assumption in \refprop{plancherel} is the main cause of the bounded assumption on activation function. Recently, \citet{sonoda2021ghosts} have extended the Plancherel formula for unbounded activation functions. Thus, it is technically possible and left for our future work to derive the lower bound for unbounded activation functions by similar arguments. %This is left for our future work.

%\paragraph{Remark 3.} 
\paragraph{Extension to Deep.}
The ridgelet theory is essentially based on the linearity of parameter distribution $\mu$ in the integral representation $S[\mu]$. But this linearity is specific to the single-hidden-layer structure. 
% Namely, 
% %in the integral representaiton,
% a single hidden layer $\sigma(a \cdot x -b)$ corresponds to a basis, and the linear output layer $\gamma(a,b)$ corresponds to the coefficient. In the case of deep neural nets (DNNs), such a simple correspondence does not hold anymore. Thus, deriving a lower bound for DNNs is not straightforward.
%One of the major difficulties is the nonlinearity of parameter distributions. 
Namely, in a DNN such as $S[\mu_2] \circ S[\mu_1]$, the $\mu_1$ (inside the $\sigma$ of $S[\mu_2]$) is no more linear. 
%Thus, the deep ridgelet transform cannot be a linear operator. 
Technically, we need a deep ridgelet theory, but there are no such things yet. We note that other theories based on the linearity of shallow networks, such as the mean field theory, also face to the same difficulty.

%\paragraph{Remark 4.}
\paragraph{Verification of Tightness.}
%Some readers may be concerned with the tightness of our lower bound, or the comparison to an upper bound w.r.t $d$ and $\lambda$ such as $\inf_{\mu \in M(d)} \| f - S[\mu] \| \lesssim \alpha(d,\lambda)$. 
%
In fact, the obtained lower bound is not tight, simply because NNs with $M(V)$ (infinitely many hidden units) are more expressive than NNs with $M(d)$ (at most $d$ hidden units). 
%Here, we note that %namely $\inf_{\mu \in M(V)} \| f - S[\mu] \|$, 
(On the other hand, for the infinitely-many-hidden-units case $M(V)$, the Pythagorean relation \eqref{eq:minimizer2} is tight.)
Nonetheless, we consider this relaxed bound meaningful because we can interpret the bound as: If the band is limited, then even if we use infinite units, the approximation power can be limited. 
%Namely, just contrary to the upper bound, the lower bound for more expressive (=infinite) networks is automatically valid for less expressive (=finite) networks. So, our scope is on finite networks.
%
%
\paragraph{Estimation of Upper Bound.} We consider it an out-of-scope because (1) estimating the approximation error with respect to \emph{finite} unit number $d$ \emph{with} bandlimiting assumption is another challenging problem, and (2) our focus is to present a non-trivial lower bound (since sometimes random nets are misunderstood to be always universal). In fact, before this study, there was no lower bound for a bandlimited network, even though it sounds reasonable when we consider the Fourier series expansion. And the difficulty why it has not been shown is the existence of null components, which Fourier series expansion does not hold.
For the case of \emph{finite} hidden units \emph{without} bandlimiting assumption, two types of upper bounds---the Jackson bound $O(d^{-s/m})$ and the Maurey-Jones-Barron (MJB) bound $O(1/\sqrt{d})$ were obtained by multiple authors in the 1990s. However, these upper bounds %essentially assumes 
are in general not sharp for bandlimiting cases.
%For example, according to (one of) the most versatile version of the MJB bound, we can derive an upper bound for bandlimited infinite-width random nets: For any bandlimited continuous NN $S[\mu]$ with $\mu \in M(V)$, we can always find a $d$-term finite net $S[\mu_d]$ with $\mu_d \in M(d)$ satisfying $\| S[\mu] - S[\mu_d] \|_2 \lesssim 1/\sqrt{d}$. But this proposition is only about the deviation between a given network $S[\mu]$ and its $d$-term discretization $S[\mu_d]$, and does not say anything about the expressive power of bandlimited NNs, e.g., whether the continuous network $S[\mu]$ (with bandlimited assumption on $\mu$) can or cannot express an arbitrary given function $f$. Thus we cannot go through Barron's argument to investigate the tightness of our lower bound.

\section{Related Work and Further Remarks}
For a whole picture, we should recall the pioneering work by Barron, Theorem~6 in \cite{Barron1993}, which is a lower bound on the best approximation error for linear combinations of any \emph{fixed} basis functions:
\vspace{-2.5ex}
\begin{align*}
    \inf_{(a_j,b_j)} \sup_{C_f \le C} \inf_{c_j} \| f - g_d \|_{L^2([0,1]^m)} \ge \frac{\kappa C}{md^{1/m}},
\end{align*}
\vspace{-4ex}

where $C_f$ is a certain complexity of function $f$,
%Their lower bound (if we state it informally) is estimated as \kappa Cm^{-1}\num^{-1/m}\kappa Cm^{-1}\num^{-1/m}, where
$\kappa$ is a universal constant not smaller than $1/(8\pi e^{\pi-1})$ (further refinements/improvements can also be found in \cite{gnecco2012comparison,kurkova2002comparison}), $m$ denotes the input dimension, $\num$ stands for the number of hidden neurons.
Barron's theoretical results, related to the so-called \emph{Kolmogorov width}, indicate that ``\emph{fixed basis function expansions must have a worst-case performance that is much worse than that which is proven to be achievable by certain adaptable basis function methods (such as neural nets)}.''
We note that \emph{neural nets with random frozen weights} is a special case of
\emph{fixed basis function expansion}.
%In theory, what is missing in Barron's bound is the distribution information of the fixed bases.
However, for fixed $C$ and a given approximation error tolerance, the estimate $\kappa Cm^{-1}\num^{-1/m}$ goes to $0$ as either $m$ or $\num$ tends to infinity; in this case, the lower bound is of impractical use to show the smaller effectiveness of fixed basis function approximation.
Similarly, \citet{Yarotsky2017} considered the problem that a deep ReLU net (not random but in which all the parameters are adaptable, without any norm constraints on the weights) approximates an $L^\infty$-Sobolev function $f \in W^{s,\infty}([0,1]^m)$. Based on covering number arguments, he proved (in Theorem~5) that if a ReLU net $\epsilon$-approximates $f$ in a unit ball, i.e. $\| f \|{_{W^{s,\infty}([0,1]^m)}} \le 1$, then the ReLU net must have at least $\num_0 = c \eps^{-m/{9}s}$ units:
\vspace{-2.5ex}
\begin{align*}
    \sup_{\| f\|_{W^{s,\infty}([0,1]^m)} \le 1} \inf_{params.} \| f - g_d \|_{L^\infty([0,1]^m)} \ge \frac{C}{(md)^{9s/m}}.
\end{align*}
\vspace{-4ex}

However, this again goes to $0$ as either $m$ or $d$ tends to infinity.
The difference lies in the assumptions on the approximator $g_d$ and approximated function $f$. The Kolmogorov width considers the setting where $g_d$ is not limited and $f$ is the worst one and thus the bound is uniform, while our result considers the setting where $g_d$ is bandlimited and $f$ is an arbitrary given one and thus the bound is pointwise.
In the context of modern deep learning theory,
\citet{Yehudai2019} and \citet{Ghorbani2019} proved (under very limited settings) that the expressive power of random nets is low, while \citet{Malach2020} proved a stronger lottery ticket hypothesis, which essentially claims that the expressive power is exceptionally high. These seemingly contradictory claims are, of course, consistent.
\citet{Yehudai2019} considered the problem that a finite-dimensional random net (FRN) approximates a single ReLU neuron and provided an approximation lower bound w.h.p. for a finite number of parameters $\num$ to conclude low expressive power.
\citet{Ghorbani2019} considered the problem that an FRN approximates a quadratic function and showed that the asymptotic approximation error does not tend to zero (Theorem~1).
Namely, these two studies focused on specific examples that FRNs \emph{cannot} easily approximate.
On the other hand, \citet{Malach2020} considered the so-called student-teacher problem in which a student FRN approximates teacher FRN, and proved that if both the student and the teacher share a common norm constraint, then
the student can $\epsilon$-approximate the teacher w.h.p., which does not contradict the previous two (and our) claims because this study focused on specific examples that FRNs \emph{can} easily approximate.
%\citet{needell2020random} proved a refined upper bound of approximation error for RVFL networks \cite{igelnik1995stochastic}.
\citet{hsu2021approximation} studied the approximation power of two-layer networks of random ReLUs, where both upper and lower-bounds for Lipschitz functions with explicit asymptotics were provided. However, the role of the hyper-parameter $\lambda$, determining the selection range of the random weights (and biases), is not considered in their main theorems, in contrast to our \refthm{mainclaim}. %Overall, 
Compared to these results, we consider the problem in which a potentially infinite-dimensional random net approximates an $L^2$-Sobolev function $f \in H^s(\Omega)$ and provide an approximation error lower bound. Thus, our results cover a wider range of functions than previous studies.
% In contrast, the traditional Kolmogorov width by Barron \cite{Barron1993} corresponds to the latter or a specified setting that approximates an infinite random net using a finite-dimensional subclass of infinite random nets. It is noteworthy that we provide the approximation lower bound concerning the smoothness of the function. Thus, our results can naturally cover the results of \cite{Ghorbani2019, Yehudai2019}.

% in that the former considers a misspecified setting where a random net approximates something other than random nets, while the latter holds for a specified setting (like the teacher-student setting), where a random net approximates another random net. This study corresponds to the former, or a misspecified setting that approximates a function of the Sobolev class using random nets.
% In contrast, the traditional Kolmogorov width by Barron \cite{Barron1993} corresponds to the latter or a specified setting that approximates an infinite random net using a finite-dimensional subclass of infinite random nets. It is noteworthy that we provide the approximation lower bound concerning the smoothness of the function. Thus, our results can naturally cover the results of \cite{Ghorbani2019, Yehudai2019}.

%\textbf{Further Comparisons.}
%}

\section{{Numerical} Experiments}\label{sec4}
In this section, we conduct some simulation studies to verify our theoretical results.
Two toy examples for 1D function regression are used in our experiments. Consistent with our theoretical analysis, the numerical simulations aim at showing how $\lambda$, which is used for the random assignment of input weights (and biases), would affect the expressive power of the random net. For this purpose, we present an intuitive illustration of the infeasibility of individual trivial settings of $\lambda$. Then, we would discuss statistically the potential relationship between $\lambda$ and the critical parameter that can determine the complexity of the target function.
We utilize the following 1D target function in the following Simulation 1 and Simulation 2.
\begin{equation*}
f(x;\sigma) \!= \!0.2\exp\!\left(\!\!-\frac{(x-0.4)^{2}}{\sigma^2}\!\right)+0.5\exp\!\left(\!\!-\frac{(x-0.6)^{2}}{\sigma^2}\!\right),
\end{equation*}
where $x\in[0,1]$, ${\sigma}>0$ is a scalar index that can determine the complexity of $f$, as mentioned in our theoretical analysis. In Simulations 1 and 2, we use the sigmoid activation function.

\textbf{Simulation 1.} We set $\sigma=0.05$ and sample 1000 instances $\{x_i,f(x_i)\}_{i=1}^{1000}$ based on equally spaced points on [0,1], then randomly and uniformly select 500 training sample and 500 test samples. We test the performance of two random networks with $\lambda=1$ and $\lambda=20$. For each case, we train the network with a different number of hidden nodes $L$, which helps with excluding the influence of $L$ to our analysis.
In \reffig{1}, we show the training and test approximation results for four different random networks, including (a) and (b) for the network built with $\lambda=1, L=100$, (c) and (d) for the network built with $\lambda=1, L=500$, (e) and (f) for the network built with $\lambda=1, L=10000$, (g) and (h) for the network built with $\lambda=20, L=200$, respectively. We observe that the random network with $\lambda=1$ cannot achieve a good approximation accuracy for this simple function approximation problem, even when the number of hidden nodes is sufficiently large. In contrast, the network with $\lambda=20$ and trained with $L=200$ demonstrates excellent learning and generalization performance. Other larger values of $\lambda$, such as $\lambda=50,100,150,200$ as we tested, have the same excellent performance on this regression task. This implies that the choice of $\lambda$ has a strong impact on the random network's expressive power, which is consistent with our theoretical results.
\begin{figure*}[htbp!]
\centering
\subfigure[{\scriptsize $\lambda=1$, $L=100$, Train}]{\includegraphics[width=0.245\textwidth]{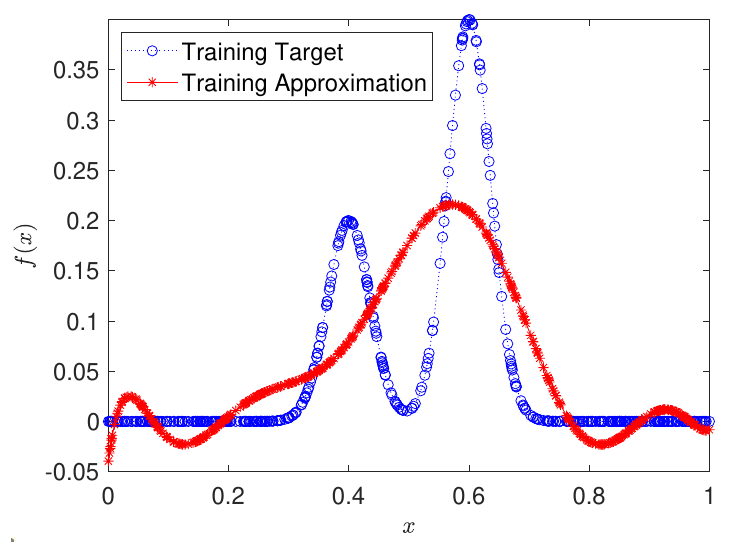}}
\subfigure[{\scriptsize$\lambda=1$, $L=100$, Test}]{\includegraphics[width=0.245\textwidth]{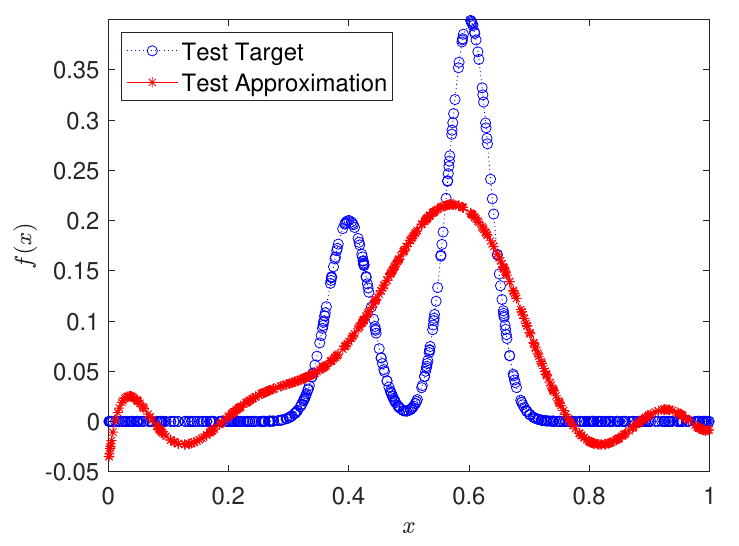}}
\subfigure[{\scriptsize$\lambda=1$, $L=500$, Train}]{\includegraphics[width=0.245\textwidth]{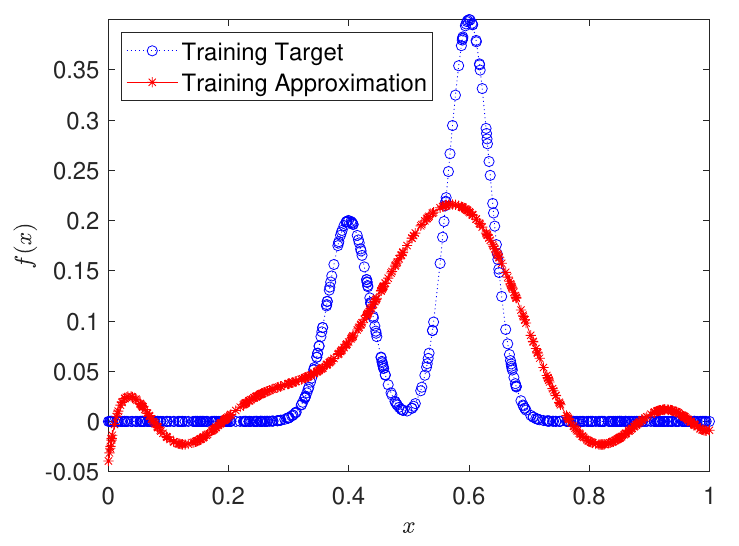}}
\subfigure[{\scriptsize$\lambda=1$, $L=500$, Test}]{\includegraphics[width=0.245\textwidth]{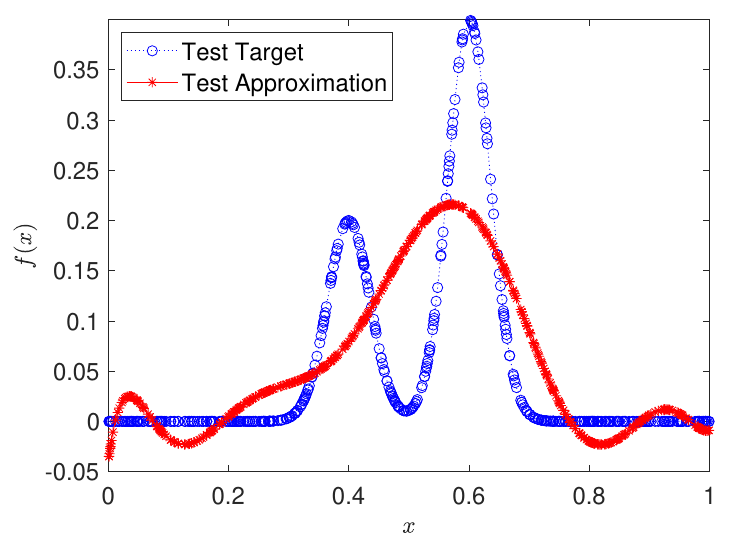}}
\subfigure[{\scriptsize$\lambda=1$, $L=10000$, Train}]{\includegraphics[width=0.245\textwidth]{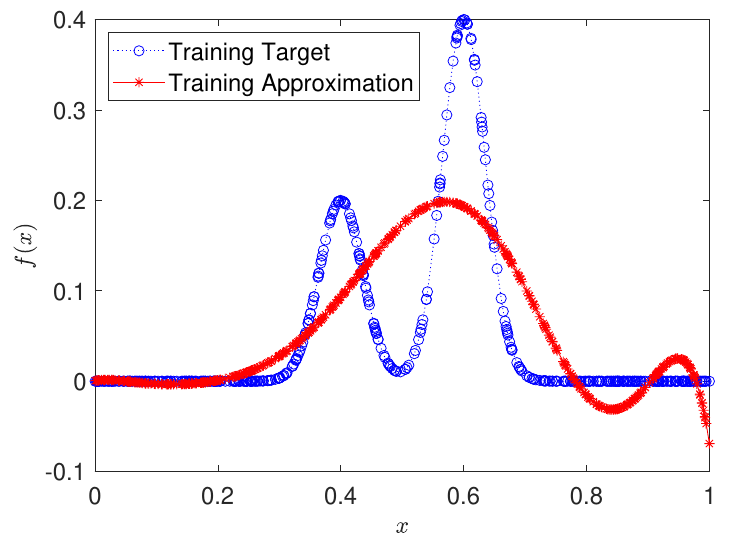}}
\subfigure[{\scriptsize$\lambda=1$, $L=10000$, Test}]{\includegraphics[width=0.245\textwidth]{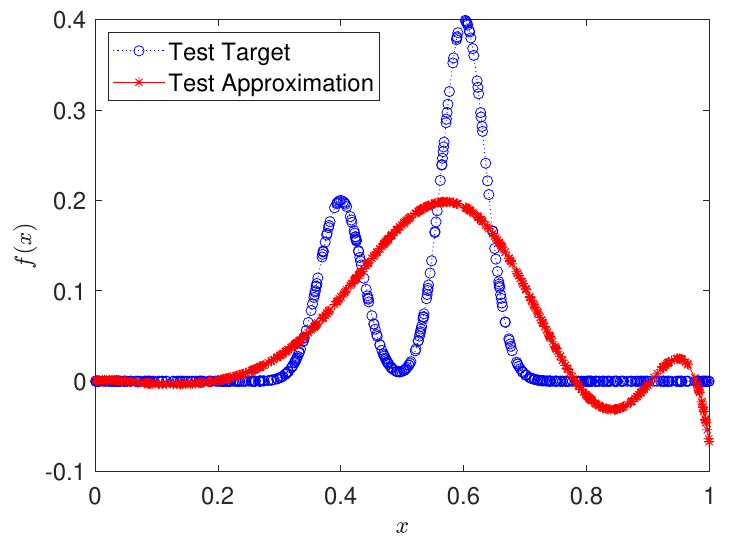}}
\subfigure[{\scriptsize$\lambda=20$, $L=200$, Train}]{\includegraphics[width=0.245\textwidth]{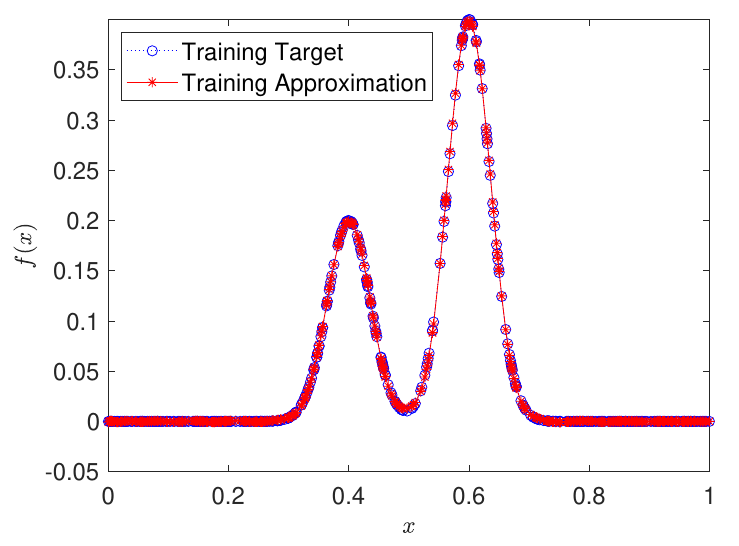}}
\subfigure[{\scriptsize$\lambda=20$, $L=200$, Test}]{\includegraphics[width=0.245\textwidth]{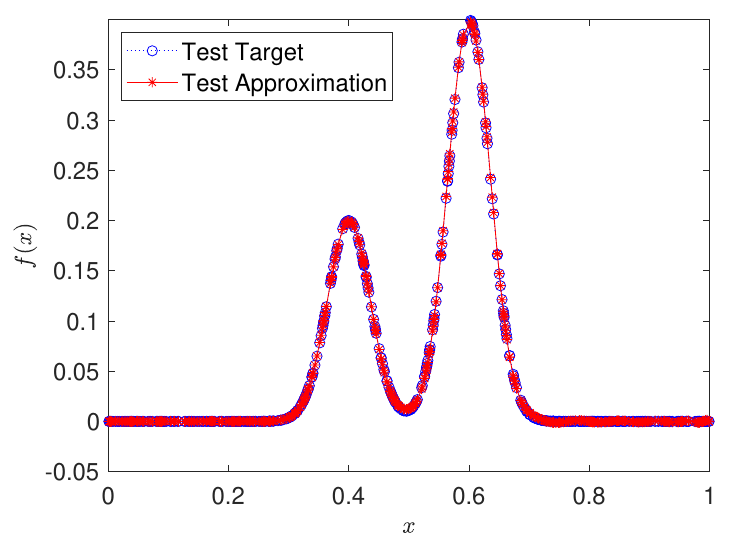}}
\caption{Performance of random nets with $\lambda=1$ and $\lambda=10$ in training and test. (a-b) $\lambda=1,L=100$. (c-d) $\lambda=1,L=500$.
 (e-f) $\lambda=1,L=10000$. (g-h) $\lambda=20,L=200$.}\label{fig:1}
\end{figure*}

\textbf{Simulation 2.} Following the intuitive investigation of the role of $\lambda$ in the expressive power of random networks in \textbf{Simulation 1}, in this part, we present more statistical results for function approximation with various pairs of ($\lambda,\sigma$) so that we can summarize a general pattern empirically. Specifically, we create different forms of target function $f(x;\sigma)$ by choosing $\sigma$ as one element of the set $\{0.01,0.05,0.1,0.5\}$, and for each regression task we build random nets with $\lambda$ taken as an element from the set $\{0.1,0.5,1,5,10,50,100,200\}$, and choose a sufficiently large $L$ (here, $L=10000$ in each case) so that we can observe the trend as $L\rightarrow +\infty$. In a similar way as in Simulation 1, we sample 1000 instances $\{x_i,f(x_i)\}_{i=1}^{1000}$ which are equally spaced points on [0,1], then randomly and uniformly select 500 training samples and 500 test samples. For each pair ($\lambda,\sigma$), we run independently 50 trials and calculate the relative training error $E_k:=\|\vec{f}-\vec{y}\|_2/\|\vec{f}\|_2$ for each trial, where $k=1,2,\ldots,50$, $\vec{f}=(f(x_1),f(x_2),\ldots,f(x_{500}))$ represents the vector of training targets, $\vec{y}=(y_1,y_2,\ldots,y_{500})$ stands for the output vector of the random network. As a matter of fact, as already shown by \reffig{1}, we only need to study the training performance to see whether a given $\lambda$ is suitable for approximating the target function produced by a given $\sigma$.
\begin{table}[htbp!]
\caption{Summary of mean relative training error for various choices of ($\lambda,\sigma$).}
\label{table:1}
\vskip 0.1in
\begin{center}
%\begin{footnotesize}
% \begin{sc}
\scriptsize
\begin{tabular}{lcccc}
\toprule[1pt]
% \hline
\multirow{2}{*}{$\lambda$}      & \multicolumn{4}{c}{Averaged Relative Training Error $E$} \\
\cline{2-5} & $\sigma=0.01$                     & $\sigma=0.05$                    & $\sigma=0.1$         & $\sigma=0.5$              \\
 \hline
$\lambda=0.1$                            & 0.9504                               &  0.6969                               &  0.3149          &
0.0026
\\
$\lambda=0.5$                            & 0.9299                               &  0.6627                               &  0.2179          &
1.0606e-04
\\
$\lambda=1$                              & 0.9188                               &  0.6546                               &  0.2089          &
1.1781e-05
\\
$\lambda=5$                              & 0.8574                               &  0.1263                               &  0.0016          &
\cellcolor{green!25}5.8661e-08
\\
$\lambda=10$                             & 0.5714                               &  0.0064                               &  \cellcolor{green!25}5.5692e-08         &   \cellcolor{green!25}4.5881e-08                      \\
$\lambda=50$                             & 0.0131                               & \cellcolor{green!25}4.4905e-08        &  \cellcolor{green!25}4.6897e-08         &  \cellcolor{green!25}4.5834e-08                            \\
$\lambda=100$                            & \cellcolor{green!25}1.9055e-06       & \cellcolor{green!25}7.5046e-08        & \cellcolor{green!25}7.2133e-08           & \cellcolor{green!25}6.8683e-08                            \\
$\lambda=200$                       & \cellcolor{green!25}1.1171e-07       & \cellcolor{green!25}1.3937e-07        & \cellcolor{green!25}1.0784e-07     & \cellcolor{green!25}1.1284e-07                            \\
\bottomrule[1pt]
\end{tabular}
%\end{sc}
%\end{footnotesize}
\end{center}
\vskip -0.1in
\end{table}
Table \ref{table:1} summarizes the averaged relative training error $E:=\sum_{k=1}^{50}E_k/50$. Note that we do not provide their standard deviations here because, compared with the average values, standard deviation values may not affect the conclusion that we are aiming to verify, as we will detail later. Table~\ref{table:1} shows how the choice of $(\sigma, \lambda)$ affects the approximation ability of random networks.
From the colored cells of the table, which values are tiny (magnitude between e-8 and e-6), we can observe that, for a target function with a smaller $\sigma$ value, we would need a larger $\lambda$ for a random net to ensure a random network to achieve an accurate approximation of the target function.
From the above simulations, we can see that the effectiveness of the approximation by random networks is constrained by both the network parameter distribution and the class of target functions. For a given learning task, there exists an appropriate range/distribution $\mathcal{D}^{*}$ (not unique), but \textbf{NOT ANY} range/distribution, such that a neural network with random weights assigned from $\mathcal{D}^{*}$ can be a universal approximator (if the number of hidden nodes is sufficiently large). Second, the $\mathcal{D}^{*}$ (for example, $[-\lambda^*,\lambda^*]$) is highly dependent upon the complexity of the target function. One needs an adequate amount of samples from the target function to provide some prior knowledge or empirical studies to discover $\mathcal{D}^{*}$.

\textbf{Simulation 3.} To further reveal the infeasibility of the trivial range [-1,1] for certain function approximation problems, we conduct similar simulations on a new target function $g(x)$, denoted as
\begin{equation*}
g(x) = 0.5\cos(22\pi x^2)+0.5x^2, \:x\in[0,1].
\end{equation*}

Mathematically, $g(x)$ is composed of two parts: $0.5\cos(22\pi x^2)$ and $0.5x^2$, which represent two completely different `modes' at distinct `frequencies', as shown in Figure \ref{fig:newadded4}.
\begin{figure}[h]%tbp!]
 \centering
 \includegraphics[width=0.5\textwidth]{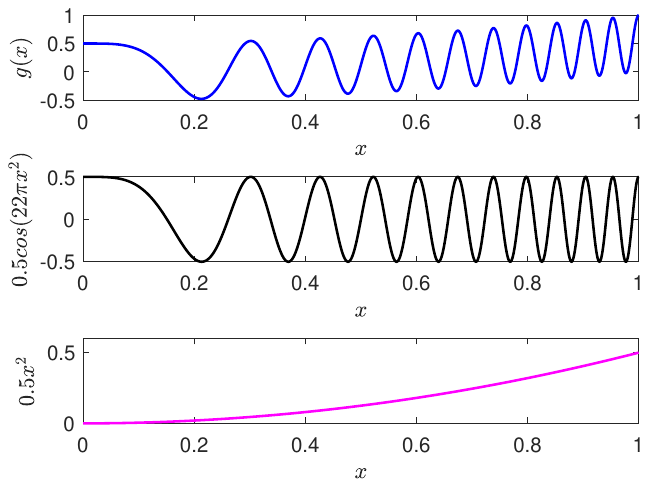}
 \vskip -0.15in
 \caption{Visualization for the target function $g(x)$ (up), its modes $0.5\cos(22\pi x^2)$ (middle) and $0.5x^2$ (bottom), respectively.}\label{fig:newadded4}
\end{figure}

We carry on the same sampling as in Simulations 1 and 2 to generate 500 training and test points on [0,1]. Here, we only consider the training performance of random nets with various choices of $\lambda$. We report the results of the comparison for $\lambda=1$ and $\lambda=100$ in \reffig{3}. We observe that the random net with $\lambda=1$ is not a universal approximator, although the number of hidden nodes is sufficiently large ($L=10,000$). The network with $\lambda=1$ can only fit the `mean' curve of the original signal and fails
to approximate the high-frequency `mode' $0.5\cos(22\pi x^2)$. On the other hand, for the second `mode' $0.5x^2$, the random net with $\lambda=1$ has great approximation performance.

As we observe the derivative  $|g'(x)|\leq 25$ in \reffig{3} (c), we conjecture that in general, the `appropriate' range of $\lambda$ is related to the magnitude of $|g'(x)|$, rather than independent of the target function class and training samples.
Moreover, a multi-scale strategy that selects random parameters from various scopes can be beneficial, especially when the target function is complicated and composed of multiple `modes.' In \reffig{3} (d), we find another interesting result that the training output of the network with 300 hidden neurons and weights (and biases) randomly chosen from [-100,100] is not significantly affected if we remove 85 hidden neurons with weights (and biases) located in the `narrow' range [-30,30]. It means, these hidden weights (and biases) as randomly assigned from [-30,30],  not to mention the ones from [-1,1], provide a little contribution to approximation universality in learning.
\begin{figure}[htbp!]
\centering
\subfigure[$\lambda=1$, $L=10000$]{\includegraphics[width=0.235\textwidth]{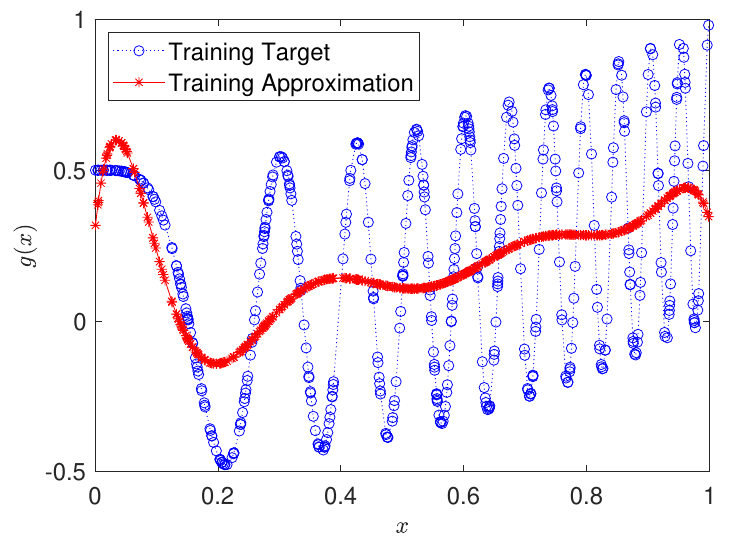}}
\subfigure[$\lambda=100$, $L=300$]{\includegraphics[width=0.235\textwidth]{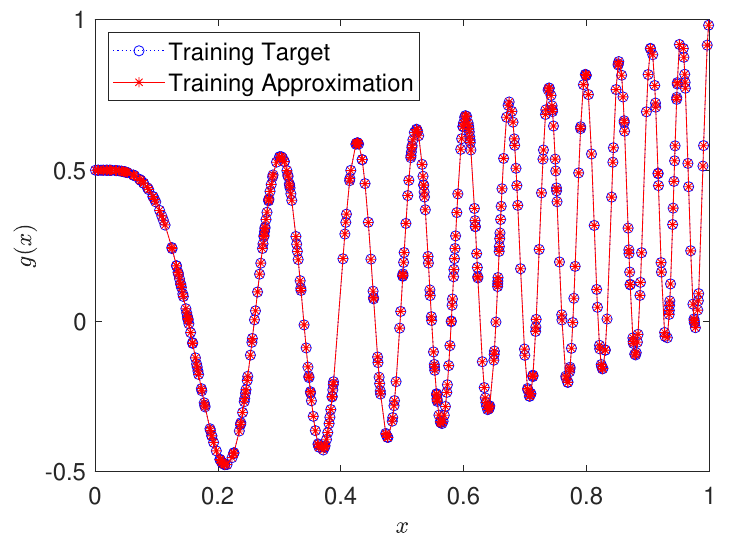}}
\subfigure[ $g'(x)$]{\includegraphics[width=0.235\textwidth]{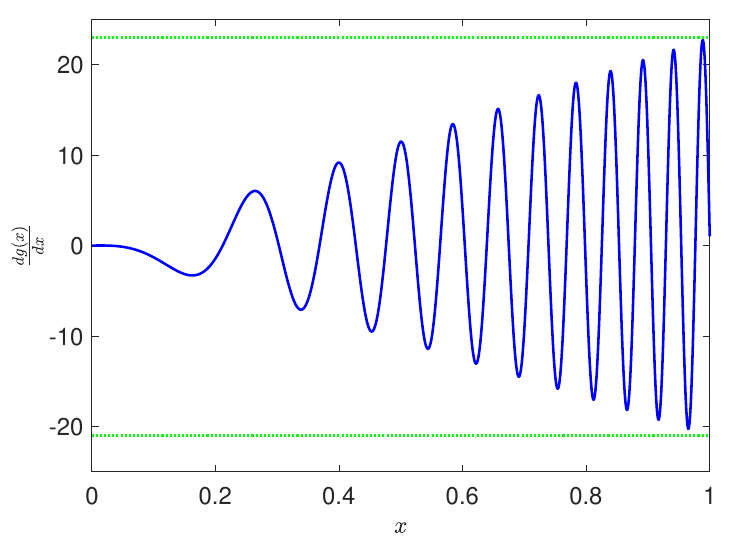}}
\subfigure[$y1$, $y2$, and $y1-y2$]{\includegraphics[width=0.235\textwidth]{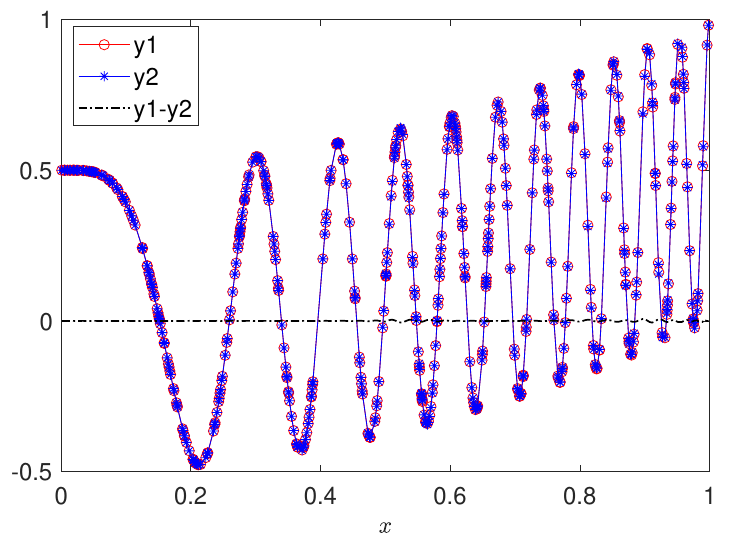}}
\caption{Performance for training results of $g(x)$: (a) $\lambda=1,L=10000$, (b) $\lambda=100, L=300$. (c) Derivative function $g'(x)$. (d) $\mathcal{N}_1$: Training approximation of $g(x)$ with hidden weights (and biases) randomly assigned from $[-100,100]$, $\mathcal{N}_2$: Training approximation of $g(x)$ with hidden weights (and biases) randomly assigned from $[-100,100]/[-30,30]$, and their numerical difference $\mathcal{N}_1-\mathcal{N}_2$.}\label{fig:3}
\end{figure}
% end of editing by Yu Guang

\textbf{Two main take-home messages:} Our experiments support our theoretical results, which send two critical messages.
(1) For a learning task, simply taking a fixed scope $[-\lambda,\lambda]$
would not make random neural nets universal approximators, if $\lambda$ is not set properly. (2) For a gaussian-type target function $f(x;\sigma)=\exp(-|x|^2/\sigma^2)$ ($\sigma>0$), which is a Sobolev function and thus meets the conditions of our main theorem,
 a large value of $\lambda$ is usually needed if $\sigma$ is small. Our empirical findings provide valuable guidance for developing algorithms for constructing random neural networks. As a practical suggestion, users utilizing random networks for data modeling should be aware that the selection of the parameter $\lambda$ greatly impacts the performance of the model. To determine an appropriate value of $\lambda$, it is recommended to conduct simulations through a trial-and-error approach. While this method is relatively straightforward to implement, it relies heavily on human intervention and is not a fully automated algorithm.

Our theoretical and empirical results indicate that randomly assigning weights from a fixed range or distribution that is independent of the training samples or prior knowledge may not be the most effective approach. Instead, it is more beneficial to explore different settings of random weights from various distributions, with the goal of expanding the function basis and increasing the ability to approximate the target function. Refer to the additional information provided in Appendix~\ref{further back}. Further simulations on a 2D example and some real-world datasets are provided in Appendix \ref{supp:furtherexp-2D} and Appdendix \ref{supp:furtherexp-real}, respectively.
% \begin{figure}[htbp!]
% \centering
% \subfigure[Visualization of the derivative function g'(x)g'(x)]{\includegraphics[width=0.48\textwidth]{fig/fig4_1}}
% \subfigure[Visualization of y1y1, y2y2, and y1-y2y1-y2]{\includegraphics[width=0.48\textwidth]{fig/fig4_2}}
% \caption{Visualization of the derivative function g'(x)g'(x) and y1y1: Training approximation of g(x)g(x) with hidden weights (and biases) randomly assigned from [-100,100], y2y2: Training approximation of g(x)g(x) with hidden weights (and biases) randomly assigned from [-100,100]/[-30,30], and their numerical difference y1-y2y1-y2.}\label{fig:4}
% \end{figure}

%-------------------------------------------------------------
%\input{literature-theory}

\section{Conclusion and Discussion}\label{sec_conc}
In this paper, we examine the lower bound on the approximation error of shallow neural networks with random weights. Specifically, we explore the impact and limitations of randomness on the network's capacity for expression. Our theoretical findings indicate that the lower bound on the approximation error of a random network may not be zero if the range/distribution of hidden parameters is not appropriately selected in advance.
Our results are based on the assumption of bandwidth limitation, which is a form of stochastic limitation that includes a finite variation, and are also valid when the proposed distribution is fully supported, such as a normal distribution.
%This finding contradicts the common saying that a shallow, random neural network is always a universal approximator regardless of the choice of hidden weights.
%
% Our study can be generally viewed as the very first step on a road that could take us to understand far better the feasibility and effectiveness of neural nets with random weights. 

Our theoretical results and empirical findings provide evidence that challenges the prevalent belief that a shallow random neural network is always a universal approximator regardless of the choice of hidden weights. This is significant as it helps researchers working with shallow neural networks and random weights to have a better understanding of the critical issues and potential drawbacks associated with randomness. Further in-depth analysis, both for deep neural networks or with tighter bounds, is expected to provide more insights. Interpretation of when and why neural networks with random weights are effective or not is essential to advance the understanding of this research topic.

\vspace{-2ex}
\section*{Acknowledgements}
Ming Li acknowledged the support from the National Natural Science Foundation of China (No. U21A20473, No. 62172370), and from the Zhejiang Provincial Natural Science Foundation (No. LY22F020004). Feilong Cao acknowledged the support from the National Natural Science Foundation of China (No. 62176244). Sho Sonoda was supported by JSPS KAKENHI 18K18113, JST CREST JPMJCR2015 and JST PRESTO JPMJPR2125. Yu Guang Wang acknowledged the support from the frontier research on the fundamentals of artificial intelligence mathematics (No. P22KN005). Jiye Liang acknowledged the support from the National Natural Science Foundation of China (No. U21A20473).

\bibliography{references}

\begin{thebibliography}{71}
\providecommand{\natexlab}[1]{#1}
\providecommand{\url}[1]{\texttt{#1}}
\expandafter\ifx\csname urlstyle\endcsname\relax
  \providecommand{\doi}[1]{doi: #1}\else
  \providecommand{\doi}{doi: \begingroup \urlstyle{rm}\Url}\fi

\bibitem[Bach(2017)]{Bach2017a}
Bach, F.
\newblock On the equivalence between kernel quadrature rules and random feature
  expansions.
\newblock \emph{Journal of Machine Learning Research}, 18\penalty0
  (1):\penalty0 714--751, 2017.

\bibitem[Barbier et~al.(2020)Barbier, Macris, Dia, and
  Krzakala]{barbier2020mutual}
Barbier, J., Macris, N., Dia, M., and Krzakala, F.
\newblock Mutual information and optimality of approximate message-passing in
  random linear estimation.
\newblock \emph{IEEE Transactions on Information Theory}, 66\penalty0
  (7):\penalty0 4270--4303, 2020.

\bibitem[Barron(1993)]{Barron1993}
Barron, A.~R.
\newblock Universal approximation bounds for superpositions of a sigmoidal
  function.
\newblock \emph{IEEE Transactions on Information Theory}, 39\penalty0
  (3):\penalty0 930--945, 1993.

\bibitem[Belkin et~al.(2019)Belkin, Hsu, Ma, and Mandal]{Belkin2019a}
Belkin, M., Hsu, D., Ma, S., and Mandal, S.
\newblock {Reconciling modern machine learning practice and the classical
  bias--variance trade-off}.
\newblock \emph{Proceedings of the National Academy of Sciences}, 116\penalty0
  (32):\penalty0 15849--15854, 2019.

\bibitem[Boutsidis et~al.(2014)Boutsidis, Zouzias, Mahoney, and
  Drineas]{boutsidis2014randomized}
Boutsidis, C., Zouzias, A., Mahoney, M.~W., and Drineas, P.
\newblock Randomized dimensionality reduction for $k$-means clustering.
\newblock \emph{IEEE Transactions on Information Theory}, 61\penalty0
  (2):\penalty0 1045--1062, 2014.

\bibitem[Cand{\`e}s(1999)]{Candes.HA}
Cand{\`e}s, E.~J.
\newblock Harmonic analysis of neural networks.
\newblock \emph{Applied and Computational Harmonic Analysis}, 6\penalty0
  (2):\penalty0 197--218, 1999.

\bibitem[Cao et~al.(2018)Cao, Wang, Ming, and Gao]{cao2018review}
Cao, W., Wang, X., Ming, Z., and Gao, J.
\newblock A review on neural networks with random weights.
\newblock \emph{Neurocomputing}, 275:\penalty0 278--287, 2018.

\bibitem[Carroll \& Dickinson(1989)Carroll and Dickinson]{Carroll.Dickinson}
Carroll, S.~M. and Dickinson, B.~W.
\newblock {Construction of neural nets using the Radon transform}.
\newblock In \emph{International Joint Conference on Neural Networks 1989},
  volume~1, pp.\  607--611. IEEE, 1989.

\bibitem[Chen \& Liu(2017)Chen and Liu]{chen2017broad}
Chen, C.~P. and Liu, Z.
\newblock Broad learning system: An effective and efficient incremental
  learning system without the need for deep architecture.
\newblock \emph{IEEE Transactions on Neural Networks and Learning Systems},
  29\penalty0 (1):\penalty0 10--24, 2017.

\bibitem[Cho \& Saul(2009)Cho and Saul]{Cho2009}
Cho, Y. and Saul, L.~K.
\newblock Kernel methods for deep learning.
\newblock In \emph{Advances in Neural Information Processing Systems}, pp.\
  342--350, 2009.

\bibitem[Daniely(2017)]{Daniely2017}
Daniely, A.
\newblock {SGD} learns the conjugate kernel class of the network.
\newblock In \emph{Advances in Neural Information Processing Systems}, pp.\
  2422--2430, 2017.

\bibitem[Donoho(2002)]{Donoho2002}
Donoho, D.~L.
\newblock {Emerging applications of geometric multiscale analysis}.
\newblock \emph{Proceedings of the ICM, Beijing 2002}, pp.\  209--233, 2002.

\bibitem[E et~al.(2019)E, Ma, and Wu]{E2019a}
E, W., Ma, C., and Wu, L.
\newblock {A priori estimates of the population risk for two-layer neural
  networks}.
\newblock \emph{Communications in Mathematical Sciences}, 17\penalty0
  (5):\penalty0 1407--1425, 2019.

\bibitem[Funahashi(1989)]{Funahashi1989}
Funahashi, K.-I.
\newblock {On the approximate realization of continuous mappings by neural
  networks}.
\newblock \emph{Neural Networks}, 2\penalty0 (3):\penalty0 183--192, jan 1989.

\bibitem[Ghorbani et~al.(2019)Ghorbani, Mei, Misiakiewicz, and
  Montanari]{Ghorbani2019}
Ghorbani, B., Mei, S., Misiakiewicz, T., and Montanari, A.
\newblock Limitations of lazy training of two-layers neural network.
\newblock In \emph{Advances in Neural Information Processing Systems}, pp.\
  9111--9121. 2019.

\bibitem[Giryes et~al.(2015)Giryes, Sapiro, and Bronstein]{Giryes2015}
Giryes, R., Sapiro, G., and Bronstein, A.~M.
\newblock Deep neural networks with random gaussian weights : A universal
  classification strategy?
\newblock \emph{IEEE Transactions on Signal Processing}, 64\penalty0
  (13):\penalty0 3444--3457, 2015.

\bibitem[Gnecco(2012)]{gnecco2012comparison}
Gnecco, G.
\newblock A comparison between fixed-basis and variable-basis schemes for
  function approximation and functional optimization.
\newblock \emph{Journal of Applied Mathematics}, 2012, 2012.

\bibitem[Gorban et~al.(2016)Gorban, Tyukin, Prokhorov, and
  Sofeikov]{gorban2016approximation}
Gorban, A.~N., Tyukin, I.~Y., Prokhorov, D.~V., and Sofeikov, K.~I.
\newblock Approximation with random bases: Pro et contra.
\newblock \emph{Information Sciences}, 364:\penalty0 129--145, 2016.

\bibitem[Hsu et~al.(2021)Hsu, Sanford, Servedio, and
  Vlatakis-Gkaragkounis]{hsu2021approximation}
Hsu, D., Sanford, C.~H., Servedio, R., and Vlatakis-Gkaragkounis, E.~V.
\newblock On the approximation power of two-layer networks of random relus.
\newblock In \emph{Proceedings of the 34th Annual Conference on Learning
  Theory}, pp.\  2423--2461. PMLR, 2021.

\bibitem[Huang et~al.(2023)Huang, Li, Cao, Fujita, Li, and Wu]{huang2022graph}
Huang, C., Li, M., Cao, F., Fujita, H., Li, Z., and Wu, X.
\newblock Are graph convolutional networks with random weights feasible?
\newblock \emph{IEEE Transactions on Pattern Analysis and Machine
  Intelligence}, 45\penalty0 (3):\penalty0 2751--2768, 2023.

\bibitem[Igelnik \& Pao(1995)Igelnik and Pao]{igelnik1995stochastic}
Igelnik, B. and Pao, Y.-H.
\newblock Stochastic choice of basis functions in adaptive function
  approximation and the functional-link net.
\newblock \emph{IEEE Transactions on Neural Networks}, 6\penalty0 (6):\penalty0
  1320--1329, 1995.

\bibitem[Irie \& Miyake(1988)Irie and Miyake]{Irie1988}
Irie, B. and Miyake, S.
\newblock {Capabilities of three-layered perceptrons}.
\newblock In \emph{IEEE International Conference on Neural Networks}, pp.\
  641--648. IEEE, 1988.

\bibitem[Ito(1991)]{Ito1991}
Ito, Y.
\newblock {Representation of functions by superpositions of a step or sigmoid
  function and their applications to neural network theory}.
\newblock \emph{Neural Networks}, 4\penalty0 (3):\penalty0 385--394, jan 1991.

\bibitem[Jacot et~al.(2018)Jacot, Gabriel, and Hongler]{Jacot2018}
Jacot, A., Gabriel, F., and Hongler, C.
\newblock Neural tangent kernel: Convergence and generalization in neural
  networks.
\newblock In \emph{Advances in Neural Information Processing Systems}, pp.\
  8571--8580, 2018.

\bibitem[Jaeger(2002)]{jaeger2002adaptive}
Jaeger, H.
\newblock Adaptive nonlinear system identification with echo state networks.
\newblock In \emph{Advances in Neural Information Processing Systems}, pp.\
  593--600, 2002.

\bibitem[Ji et~al.(2020)Ji, Telgarsky, and Xian]{Ji2020}
Ji, Z., Telgarsky, M., and Xian, R.
\newblock {Neural tangent kernels, transportation mappings, and universal
  approximation}.
\newblock In \emph{International Conference on Learning Representations}, 2020.

\bibitem[Johnson \& Lindenstrauss(1984)Johnson and
  Lindenstrauss]{johnson1984extensions}
Johnson, W.~B. and Lindenstrauss, J.
\newblock Extensions of lipschitz mappings into a hilbert space.
\newblock \emph{Contemporary Mathematics}, 26:\penalty0 189--206, 1984.

\bibitem[Kainen et~al.(2013)Kainen, K\r{u}rkov\'{a}, and
  Sanguineti]{kainen.survey}
Kainen, P.~C., K\r{u}rkov\'{a}, V., and Sanguineti, M.
\newblock Approximating multivariable functions by feedforward neural nets.
\newblock In Bianchini, M., Maggini, M., and Jain, L.~C. (eds.), \emph{Handbook
  on Neural Information Processing}, volume~49 of \emph{Intelligent Systems
  Reference Library}, pp.\  143--181. Springer Berlin Heidelberg, 2013.

\bibitem[Kleyko et~al.(2021)Kleyko, Kheffache, Frady, Wiklund, and
  Osipov]{kleyko2020density}
Kleyko, D., Kheffache, M., Frady, E.~P., Wiklund, U., and Osipov, E.
\newblock Density encoding enables resource-efficient randomly connected neural
  networks.
\newblock \emph{IEEE Transactions on Neural Networks and Learning Systems},
  32\penalty0 (8):\penalty0 3777--3783, 2021.

\bibitem[Klusowski \& Barron(2018)Klusowski and Barron]{Klusowski2018a}
Klusowski, J.~M. and Barron, A.~R.
\newblock {Approximation by Combinations of {ReLU} and Squared {ReLU} Ridge
  Functions with $\ell^1$ and $\ell^0$ Controls}.
\newblock \emph{IEEE Transactions on Information Theory}, 64\penalty0
  (12):\penalty0 7649--7656, 2018.

\bibitem[Kostadinova et~al.(2014)Kostadinova, Pilipovi{\'{c}}, Saneva, and
  Vindas]{Kostadinova2014}
Kostadinova, S., Pilipovi{\'{c}}, S., Saneva, K., and Vindas, J.
\newblock {The ridgelet transform of distributions}.
\newblock \emph{Integral Transforms and Special Functions}, 25\penalty0
  (5):\penalty0 344--358, 2014.
\newblock \doi{10.1080/10652469.2013.853057}.

\bibitem[K\r{u}rkov{\'a} \& Sanguineti(2002)K\r{u}rkov{\'a} and
  Sanguineti]{kurkova2002comparison}
K\r{u}rkov{\'a}, V. and Sanguineti, M.
\newblock Comparison of worst case errors in linear and neural network
  approximation.
\newblock \emph{IEEE Transactions on Information Theory}, 48\penalty0
  (1):\penalty0 264--275, 2002.

\bibitem[Lee et~al.(2017)Lee, Ge, Ma, Risteski, and Arora]{Lee2017}
Lee, H., Ge, R., Ma, T., Risteski, A., and Arora, S.
\newblock {On the ability of neural nets to express distributions}.
\newblock In \emph{Proceedings of 30th Annual Conference on Learning Theory},
  pp.\  1--26, 2017.

\bibitem[Li \& Wang(2017)Li and Wang]{li2017insights}
Li, M. and Wang, D.
\newblock Insights into randomized algorithms for neural networks: Practical
  issues and common pitfalls.
\newblock \emph{Information Sciences}, 382:\penalty0 170--178, 2017.

\bibitem[Li \& Wang(2021)Li and Wang]{li20192}
Li, M. and Wang, D.
\newblock {2-D} stochastic configuration networks for image data analytics.
\newblock \emph{IEEE Transactions on Cybernetics}, 51\penalty0 (1):\penalty0
  359--372, 2021.

\bibitem[Li et~al.(2019)Li, Huang, and Wang]{li2019robust}
Li, M., Huang, C., and Wang, D.
\newblock Robust stochastic configuration networks with maximum correntropy
  criterion for uncertain data regression.
\newblock \emph{Information Sciences}, 473:\penalty0 73--86, 2019.

\bibitem[Liu et~al.(2021)Liu, Huang, Chen, and Suykens]{liu2020random}
Liu, F., Huang, X., Chen, Y., and Suykens, J.~A.
\newblock Random features for kernel approximation: A survey in algorithms,
  theory, and beyond.
\newblock \emph{IEEE Transactions on Pattern Analysis and Machine
  Intelligence}, 2021.

\bibitem[Louart et~al.(2018)Louart, Liao, Couillet, et~al.]{louart2018random}
Louart, C., Liao, Z., Couillet, R., et~al.
\newblock A random matrix approach to neural networks.
\newblock \emph{The Annals of Applied Probability}, 28\penalty0 (2):\penalty0
  1190--1248, 2018.

\bibitem[Luko{\v{s}}evi{\v{c}}Ius \& Jaeger(2009)Luko{\v{s}}evi{\v{c}}Ius and
  Jaeger]{lukovsevivcius2009reservoir}
Luko{\v{s}}evi{\v{c}}Ius, M. and Jaeger, H.
\newblock Reservoir computing approaches to recurrent neural network training.
\newblock \emph{Computer Science Review}, 3\penalty0 (3):\penalty0 127--149,
  2009.

\bibitem[Malach et~al.(2020)Malach, Yehudai, Shalev-Shwartz, and
  Shamir]{Malach2020}
Malach, E., Yehudai, G., Shalev-Shwartz, S., and Shamir, O.
\newblock Proving the lottery ticket hypothesis: Pruning is all you need.
\newblock In \emph{Proceedings of 37th International Conference on Machine
  Learning}, pp.\  6682--6691, 2020.

\bibitem[Mallat(2009)]{Mallat2009}
Mallat, S.
\newblock \emph{{A Wavelet Tour of Signal Processing, Third Edition: The Sparse
  Way}}.
\newblock Academic Press, 2009.

\bibitem[Minsky \& Papert(1988)Minsky and Papert]{minsky1988perceptrons}
Minsky, M. and Papert, S.
\newblock \emph{Perceptrons}.
\newblock MIT press, 1988.

\bibitem[Mongia et~al.(2016)Mongia, Kumar, Erraqabi, and
  Bengio]{mongia2016random}
Mongia, M., Kumar, K., Erraqabi, A., and Bengio, Y.
\newblock On random weights for texture generation in one layer neural
  networks.
\newblock \emph{arXiv preprint arXiv:1612.06070}, 2016.

\bibitem[Murata(1996)]{murata1996integral}
Murata, N.
\newblock An integral representation of functions using three-layered networks
  and their approximation bounds.
\newblock \emph{Neural Networks}, 9\penalty0 (6):\penalty0 947--956, 1996.

\bibitem[Neal(1996)]{Neal1996}
Neal, R.~M.
\newblock \emph{Bayesian Learning for Neural Networks}.
\newblock Lecture Notes in Statistics. Springer-Verlag New York, 1996.

\bibitem[Ongie et~al.(2020)Ongie, Willett, Soudry, and Srebro]{Ongie2020}
Ongie, G., Willett, R., Soudry, D., and Srebro, N.
\newblock A function space view of bounded norm infinite width relu nets: The
  multivariate case.
\newblock In \emph{International Conference on Learning Representations}, 2020.

\bibitem[Pao et~al.(1994)Pao, Park, and Sobajic]{Pao1994}
Pao, Y.~H., Park, G.~H., and Sobajic, D.~J.
\newblock Learning and generalization characteristics of the random vector
  functional-link net.
\newblock \emph{Neurocomputing}, 6\penalty0 (2):\penalty0 163--180, 1994.

\bibitem[Parhi \& Nowak(2021)Parhi and Nowak]{Parhi2020}
Parhi, R. and Nowak, R.~D.
\newblock Banach space representer theorems for neural networks and ridge
  splines.
\newblock \emph{Journal of Machine Learning Research}, 22\penalty0
  (43):\penalty0 1--40, 2021.

\bibitem[Pennington \& Worah(2017)Pennington and
  Worah]{pennington2017nonlinear}
Pennington, J. and Worah, P.
\newblock Nonlinear random matrix theory for deep learning.
\newblock In \emph{Advances in Neural Information Processing Systems}, pp.\
  2637--2646, 2017.

\bibitem[Rahimi \& Recht(2008{\natexlab{a}})Rahimi and Recht]{rahimi2008random}
Rahimi, A. and Recht, B.
\newblock Random features for large-scale kernel machines.
\newblock In \emph{Advances in Neural Information Processing Systems}, pp.\
  1177--1184, 2008{\natexlab{a}}.

\bibitem[Rahimi \& Recht(2008{\natexlab{b}})Rahimi and
  Recht]{rahimi2008uniform}
Rahimi, A. and Recht, B.
\newblock Uniform approximation of functions with random bases.
\newblock In \emph{Proceedings of the 46th Annual Allerton Conference on
  Communication, Control, and Computing}, pp.\  555--561, 2008{\natexlab{b}}.

\bibitem[Rahimi \& Recht(2009)Rahimi and Recht]{rahimi2009weighted}
Rahimi, A. and Recht, B.
\newblock Weighted sums of random kitchen sinks: {Replacing} minimization with
  randomization in learning.
\newblock In \emph{Advances in Neural Information Processing Systems}, pp.\
  1313--1320, 2009.

\bibitem[Rosenblatt(1958)]{rosenblatt1958perceptron}
Rosenblatt, F.
\newblock The perceptron: a probabilistic model for information storage and
  organization in the brain.
\newblock \emph{Psychological Review}, 65\penalty0 (6):\penalty0 386, 1958.

\bibitem[Rubin(1998)]{Rubin1998}
Rubin, B.
\newblock {The Calder{\'{o}}n reproducing formula, windowed X-ray transforms,
  and radon transforms in $L^p$-spaces}.
\newblock \emph{Journal of Fourier Analysis and Applications}, 4\penalty0
  (2):\penalty0 175--197, 1998.

\bibitem[Savarese et~al.(2019)Savarese, Evron, Soudry, and
  Srebro]{Savarese2019}
Savarese, P., Evron, I., Soudry, D., and Srebro, N.
\newblock How do infinite width bounded norm networks look in function space?
\newblock In \emph{Proceedings of the 32nd Annual Conference on Learning
  Theory}, pp.\  2667--2690, 2019.

\bibitem[Saxe et~al.(2011)Saxe, Koh, Chen, Bhand, Suresh, and
  Ng]{saxe2011random}
Saxe, A.~M., Koh, P.~W., Chen, Z., Bhand, M., Suresh, B., and Ng, A.~Y.
\newblock On random weights and unsupervised feature learning.
\newblock In \emph{Proceedings of the 28th International Conference on Machine
  Learning}, pp.\  1089--1096, 2011.

\bibitem[Scardapane \& Wang(2017)Scardapane and Wang]{scardapane2017randomness}
Scardapane, S. and Wang, D.
\newblock Randomness in neural networks: an overview.
\newblock \emph{Wiley Interdisciplinary Reviews: Data Mining and Knowledge
  Discovery}, 7\penalty0 (2):\penalty0 e1200, 2017.

\bibitem[Schmidt et~al.(1992)Schmidt, Kraaijveld, and Duin]{Schmidt1992}
Schmidt, W.~F., Kraaijveld, M.~A., and Duin, R.~P.
\newblock Feedforward neural networks with random weights.
\newblock In \emph{Proceedings of the 11th International Conference on Pattern
  Recognition}, pp.\  1--4, 1992.

\bibitem[Sonoda(2019)]{Sonoda2019}
Sonoda, S.
\newblock Fast approximation and estimation bounds of kernel quadrature for
  infinitely wide models.
\newblock \emph{arXiv preprint arXiv:1902.00648}, 2019.

\bibitem[Sonoda \& Murata(2014)Sonoda and Murata]{sonoda2014sampling}
Sonoda, S. and Murata, N.
\newblock Sampling hidden parameters from oracle distribution.
\newblock In \emph{Proceedings of International Conference on Artificial Neural
  Networks}, pp.\  539--546. Springer International Publishing, 2014.

\bibitem[Sonoda \& Murata(2017)Sonoda and Murata]{Sonoda2015acha}
Sonoda, S. and Murata, N.
\newblock {Neural network with unbounded activation functions is universal
  approximator}.
\newblock \emph{Applied and Computational Harmonic Analysis}, 43\penalty0
  (2):\penalty0 233--268, 2017.

\bibitem[Sonoda et~al.(2018)Sonoda, Ishikawa, Ikeda, Hagihara, Sawano,
  Matsubara, and Murata]{Sonoda2018a}
Sonoda, S., Ishikawa, I., Ikeda, M., Hagihara, K., Sawano, Y., Matsubara, T.,
  and Murata, N.
\newblock The global optimum of shallow neural network is attained by ridgelet
  transform.
\newblock \emph{arXiv preprint arXiv:1805.07517}, 2018.

\bibitem[Sonoda et~al.(2021{\natexlab{a}})Sonoda, Ishikawa, and
  Ikeda]{Sonoda2021aistats}
Sonoda, S., Ishikawa, I., and Ikeda, M.
\newblock {Ridge Regression with Over-Parametrized Two-Layer Networks Converge
  to Ridgelet Spectrum}.
\newblock In \emph{Proceedings of The 24th International Conference on
  Artificial Intelligence and Statistics (AISTATS) 2021}, volume 130, pp.\
  2674--2682. PMLR, 2021{\natexlab{a}}.

\bibitem[Sonoda et~al.(2021{\natexlab{b}})Sonoda, Ishikawa, and
  Ikeda]{sonoda2021ghosts}
Sonoda, S., Ishikawa, I., and Ikeda, M.
\newblock Ghosts in neural networks: Existence, structure and role of
  infinite-dimensional null space.
\newblock \emph{arXiv preprint arXiv:2106.04770}, 2021{\natexlab{b}}.

\bibitem[Starck et~al.(2010)Starck, Murtagh, and Fadili]{Starck2010}
Starck, J.-L., Murtagh, F., and Fadili, J.~M.
\newblock {The ridgelet and curvelet transforms}.
\newblock In \emph{Sparse Image and Signal Processing: Wavelets, Curvelets,
  Morphological Diversity}, pp.\  89--118. Cambridge University Press, 2010.
\newblock \doi{10.1017/CBO9780511730344.006}.

\bibitem[Suzuki(2018)]{Suzuki2017}
Suzuki, T.
\newblock Fast generalization error bound of deep learning from a kernel
  perspective.
\newblock In \emph{Proceedings of the 21st International Conference on
  Artificial Intelligence and Statistics}, pp.\  1397--1406, 2018.

\bibitem[Wang \& Li(2017{\natexlab{a}})Wang and Li]{wang2017robust}
Wang, D. and Li, M.
\newblock Robust stochastic configuration networks with kernel density
  estimation for uncertain data regression.
\newblock \emph{Information Sciences}, 412:\penalty0 210--222,
  2017{\natexlab{a}}.

\bibitem[Wang \& Li(2017{\natexlab{b}})Wang and Li]{wang2017stochastic}
Wang, D. and Li, M.
\newblock Stochastic configuration networks: Fundamentals and algorithms.
\newblock \emph{IEEE Transactions on Cybernetics}, 47\penalty0 (10):\penalty0
  3466--3479, 2017{\natexlab{b}}.

\bibitem[Yarotsky(2017)]{Yarotsky2017}
Yarotsky, D.
\newblock {Error bounds for approximations with deep ReLU networks}.
\newblock \emph{Neural Networks}, 94:\penalty0 103--114, 2017.

\bibitem[Yehudai \& Shamir(2019)Yehudai and Shamir]{Yehudai2019}
Yehudai, G. and Shamir, O.
\newblock On the power and limitations of random features for understanding
  neural networks.
\newblock In \emph{Advances in Neural Information Processing Systems}, pp.\
  6594--6604. 2019.

\bibitem[Zhang et~al.(2011)Zhang, Miller, and Wang]{zhang2011nonlinear}
Zhang, B., Miller, D.~J., and Wang, Y.
\newblock Nonlinear system modeling with random matrices: echo state networks
  revisited.
\newblock \emph{IEEE Transactions on Neural Networks and Learning Systems},
  23\penalty0 (1):\penalty0 175--182, 2011.

\end{thebibliography}
\bibliographystyle{icml2023}

%%%%%%%%%%%%%%%%%%%%%%%%%%%%%%%%%%%%%%%%%%%%%%%%%%%%%%%%%%%%%%%%%%%%%%%%%%%%%%%
%%%%%%%%%%%%%%%%%%%%%%%%%%%%%%%%%%%%%%%%%%%%%%%%%%%%%%%%%%%%%%%%%%%%%%%%%%%%%%%
% APPENDIX
%%%%%%%%%%%%%%%%%%%%%%%%%%%%%%%%%%%%%%%%%%%%%%%%%%%%%%%%%%%%%%%%%%%%%%%%%%%%%%%
%%%%%%%%%%%%%%%%%%%%%%%%%%%%%%%%%%%%%%%%%%%%%%%%%%%%%%%%%%%%%%%%%%%%%%%%%%%%%%%
\newpage
\appendix
\onecolumn

\section{Further Background}\label{further back}
The initial motivation of this work comes from the comments posted by  Yann LeCun\footnote{\url{https://www.facebook.com/yann.lecun/posts/10152872571572143}}, where some truth background behind randomness in neural nets was briefly revisited.
We {do} not only agree with Yann's comments after we conduct a comprehensive literature review for this line of research but also technically question the feasibility and effectiveness of the ``neural nets with random weights'' (with certain controversial name/term),  since many researchers {found} empirically that, in some cases, although not always, random models with inappropriate setting of the random parameters lead to unstable or poor results .

Overall, that motivates us to investigate two pressing, however, puzzling questions: (1) Can we guarantee that {a} random neural {net} model with hidden parameters chosen from a fixed range, for example, a trivial case {obtained} by letting $\lambda=1$, is a universal approximator? (2) Given a target function $f$ with specified complexity, what is the relationship between an appropriate setting of $\lambda$ (that can lead to a universal approximator in the sense of probability) and the smoothness of $f$?  Though we raise these questions,  our intention is not to make any judgement on or get involved in the controversial name towards this direction. Instead, we present our current study along the right track of neural nets with random weights (or random neural nets, random nets), with particular concerns on the theoretical aspects, aiming to provide some new insights into answering the above questions.

The appearance of randomness in neural networks can trace back to the original Rosenblatt's perceptron \cite{rosenblatt1958perceptron}, where the first layer is randomly connected and later Minsky and Papert's Gamba perceptron \cite{minsky1988perceptrons} whose first layer is a bunch of linear threshold units. In early 1990s, researchers made random training methods/models reification, for example, {with} single hidden layer feedforward networks ({SLFNs}) with random weights \cite{Schmidt1992} and random vector functional-link (RVFL) networks \cite{Pao1994}. Algorithmically, they performed the two steps (mentioned at the beginning of the introduction section) to build the randomized learner model. However, the approximation errors of the {resulting models} are bounded in the statistical sense \cite{igelnik1995stochastic}, implying that preferable approximation performance is not guaranteed for every random assignment of the hidden parameters if the re-given probability distribution $Q(a; b)$ is not appropriately chosen \cite{gorban2016approximation,li2017insights}. In contrast,  the stochastic configuration networks \cite{wang2017stochastic} can ensure universal {approximation} by enforcing certain constraints on the random assignment of the hidden parameters, rather than using the purely random way as the ``good'' probability distribution $Q^{*}(a; b)$ is unknown and data-dependent. \citet{sonoda2014sampling} proposed the sampling regression learning method by introducing a nonparametric probability distribution of the hidden parameters of SLFNs, and fitting the output parameters via ordinary linear regression. \citet{kleyko2020density} proposed to represent input features of {RVFLs via} density-based encoding, {which is} widely known in the area of stochastic computing, and {used} the operations of binding and bundling from the area of hyper-dimensional computing for obtaining the activations of the hidden neurons. The framework of a broad learning system \cite{chen2017broad} performs in the manner of a flat network, in which the original inputs are transferred and placed as the ``mapped features'' in feature nodes and the structure is expanded in a wide sense in the ``enhancement nodes.''

Although we only pay attention to shallow NNRWs with Step I and Step II (mentioned in the introduction), some other techniques/models using randomness are still worth mentioning here, aiming to present the engaging readers with a big picture of {this} line of research. For instance, the use of randomness in deep neural nets is also concerned in terms of different viewpoints. \citet{mongia2016random} showed that simple single-layer CNNs with random filters could serve as the basis for excellent texture synthesis models.  \citet{saxe2011random} observed that the results of {a} learner based on random weights are comparable to that after regular pre-training and fine-tuning processes. \citet{Giryes2015} showed that under certain conditions, DNNs with random {gaussian} weights could perform a stable embedding of the original data, permitting a stable recovery of the data from the features represented by the network. Reservoir computing, a new paradigm to use recurrent neural networks with fixed and randomly generated weights, has also been widely adopted in-stream data modeling tasks \cite{jaeger2002adaptive,lukovsevivcius2009reservoir,zhang2011nonlinear}. Kernel approximation with random features \cite{rahimi2008random,rahimi2008uniform, rahimi2009weighted} can also be viewed as a random training method as its primary philosophy is mapping the input data to {a} randomized low-dimensional feature space and then applying existing fast linear methods. See the recent survey paper \cite{liu2020random}. On the other hand, random projections are well established and commonly used for dimensionality reduction \cite{boutsidis2014randomized,barbier2020mutual}. Here, one utilizes a random matrix to project input patterns from a high-dimensional space to a lower-dimensional representation such that distances between these patterns are preserved with high accuracy, as stated in Johnson-Lindenstrauss lemma \cite{johnson1984extensions}. 
%In addition, the spectral properties of random neural nets have been extensively studied in recent years \cite{liao2018spectrum,liao2020random,ma2020slow} % I (Sho) have commented out them because these are essentially about eigenspectrum of kernels defined by random NNs, while we only consider ridgelet spectrum of NNs

\section{Integral Representation Theory and Ridgelet Transform}\label{sec:int_rep_theory}
\subsection{Background}
The \emph{ridgelet transform} has been independently discovered by \citet{murata1996integral}, \citet{Candes.HA}, and \citet{Rubin1998} during 1996--1998 as a `harmonic analysis of neural networks'.
This is a path-breaking study, not only in the neural network field, but also in the sparse coding theory (see overviews by \citet{Donoho2002} and by \citet{Starck2010}).
The ridgelet transform has been extended to Schwartz distributions by \citet{Kostadinova2014}, and to non-integrable activation functions such as ReLU by \citet{Sonoda2015acha}.
The \emph{integral representation of a neural network} had been developed before the ridgelet transform. (Recall that the ridgelet transform $R$ is a right inverse operator of the integral representation operator $S$. Thus, we can analyze $S$ without knowing $R$.) For example, \citet{Irie1988}, \citet{Funahashi1989} and \citet{Barron1993} used Fourier transform as an integral representation to prove the UAP. \citet{Carroll.Dickinson} and \citet{Ito1991} used Radon transform. In particular, the so-called \emph{Barron class} (proposed in \citet{Barron1993}) characterizes the functions that neural networks can effectively approximate. The effectiveness here is quantified as \emph{Barron's bound}, a dimension-free approximation upper bound (see the overview by \citet{kainen.survey}).
The original Barron's theory excludes ReLU, and the upper bound is in general not tight. Thus, many authors \citep{Klusowski2018a,Lee2017,Sonoda2019,E2019a,Savarese2019,Ji2020,Ongie2020,Parhi2020} have improved and developed Barron-like theories for ReLU nets.
{It} is notable that \citet{Ongie2020} and \citet{Parhi2020} have employed the Radon transform and developed some representer theorems. The novelty of this study in the integral representation literature is in the estimation of lower bounds.
%\switch{that, besides the lower bounds, we have established the Plancherel theorem and a new reconstruction formula for kk-homogeneous activation functions such as ReLU.}{}
%%%%%%%%%%%%%%%%%%%%%%%%%%%%%%%%%%%%%%%%%%%%%%%%%%%%%%%%%%%%%%%%%%%%%%%%%%%%%%%
%\\
%\newline
\subsection{Quick Overview}
We explain the integral representation theory established in \cite{Sonoda2015acha}
showing a few new results.
In order to avoid confusion, we use two symbols $\hat{\cdot}$ and $\cdot^\sharp$ for $m$-dimensional and $1$-dimensional Fourier transforms respectively, namely,
\begin{align*}
    \widehat{f}(\xi) &:= \int_{\RR^m} f(x) e^{-i\xi\cdot x} \dd x, \quad f \in L^2(\RR^m), \xi \in \RR^m; \\
    \sigma^\sharp(\omega) &:= \int_{\RR} \sigma(t) e^{-i\omega t} \dd t, \quad \sigma \in L^2(\RR), \omega \in \RR.
\end{align*}

%\textbf{Function spaces.}
%Let PP and QQ be probability distributions on \RR^m\RR^m.
Let $P$ be a Radon measure on $\RR^m$.
We consider two Hilbert spaces $\F = L^2(P)$ and $\G = L^2(\RR^m\times\RR)$ associated with the following inner products:
\begin{align*}
    \iprod{ f,g }_\F &:= \int_{\RR^m} f(x) \overline{g(x)} \dd P(x), \\
    \iprod{ \mud, \nud }_{\G} &:= \int_{\RR^m \times \RR} \mud(a,b) \overline{\nud(a,b)} \dd a \dd b, %, \\
\end{align*}
and the Banach space $\M$ of the finite Radon measures on $\RR^m\times\RR$ equipped with the total variation norm $\| \cdot \|_{TV}$.

%\textbf{Integral representation SS.}
\begin{dfn}[Integral representation $S$]
Fix any function $\sigma:\RR\to\CC$ and measure $\mu$ on $\RR^m\times\RR$, we define the integral representation as
\begin{align*}
    S[\mu](x)
    &:= \int_{\RR^m \times \RR} \sigma(a \cdot x - b ) \dd \mu(a,b), \quad x \in \RR^m. % \\
%    &= \int_{\RR^m \times \RR} \partial_b^k \mu(a,b) \partial_b^k \sigma(a \cdot x - b ) \dd a \dd b.
\end{align*}
With a slight abuse of notation, when the measure $\mu$ has a density $\phi \in L^1(\RR^m\times\RR)$, we write $S[\phi]$ instead of $S[\phi \dd a \dd b]$.
\end{dfn}

\begin{prop}[Fourier expression of $S$]
\begin{align*}
    S[\mu](x)
    &= \frac{1}{2 \pi} \int_{\RR^m\times\RR} \mu^\sharp(a,\omega) \sigma^\sharp(\omega) \dd \omega e^{i \omega a \cdot x} \dd a.
\end{align*}
\end{prop}
\begin{proof}
Since $\sigma(a \cdot x - b) = \frac{1}{2\pi} \int_\RR \sigma^\sharp(\omega) e^{i\omega(a\cdot x-b)} \dd \omega$,
\begin{align*}
    S[\mu](x)
    &= \frac{1}{2\pi} \int_{\RR^m \times \RR} \sigma^\sharp(\omega) e^{i \omega (a \cdot x - b)} \dd \omega \dd \mu(a,b) \\
    &= \frac{1}{2\pi} \int_{\RR^m \times \RR} \mu{^\sharp}(a,\omega) \sigma^\sharp(\omega) e^{i \omega a \cdot x} \dd \omega \dd a. \qedhere
\end{align*}
\end{proof}

%\textbf{Boundedness (Lipschitz continuity) of SS.}
\begin{prop}[Boundedness (Lipschitz continuity) of $S:\M\to\F$]
We write $\sigma_{a,b}(x) := \sigma(a\cdot x - b)$.
%Provided that \sigma\sigma is bounded by a constant C_\sigma := \| \sigma \|_{L^\infty(\RR)}C_\sigma := \| \sigma \|_{L^\infty(\RR)} \red{and PP is finite with volume C_P := P(\RR^m)C_P := P(\RR^m)},
Provided that the constant $C_{\sigma,P}^2 := \sup_{a,b} \| \sigma_{a,b} \|_{\F}^2$ exists finite,
%then both S:\M\to\FS:\M\to\F and S:\G\to\FS:\G\to\F are bounded, or equivalently, Lipschitz continuous: For any \mu \in \M\mu \in \M and \phi \in \G\phi \in \G, we have respectively
then $S:\M\to\F$ is bounded, or equivalently, Lipschitz continuous: For any $\mu \in \M$, we have
\begin{align*}
    \| S[\mu] \|_{\F}^2
    &\le \int \left( \int |\sigma_{a,b}(x)| \dd |\mu|(a,b) \right)^2 \dd P(x)
    \le \| \mu \|_{TV}^2 \int \int |\sigma_{a,b}(x)|^2 \frac{\dd |\mu|(a,b)}{\|\mu\|_{TV}} \dd P(x)\le C_{\sigma,P}^2 \| \mu \|_{TV}^2.
%    \le C_\sigma^2 C_P \| \mu \|_{TV}^2.
    % \| S[\phi] \|_{\F}^2
    % &\le \int \left( \int |\phi(a,b)| |\sigma_{a,b}(x)| \dd a \dd b \right)^2 \dd P(x) \notag \\
    % &\le \int \|\phi\|_\G^2 \left( \int \sigma_{a,b}(x) \dd a \dd b \right)^2 \dd P(x) \notag
    % \le C_\sigma^2 C_P \| \phi \|_{\G}^2.
\end{align*}
\end{prop}
The boundedness of $S$ is a sufficient condition for the optimization problem to be well-defined, in the sense that $S(\M) \subset \F$.
Hence, unless otherwise noted, we always assume that $C_{\sigma,P} < \infty$.

%\textbf{Adjoint operator $S_P^*$.}
\begin{prop}[Adjoint operator $S_P^*$]
For $S:\G\to\F$, the adjoint operator $S_P^*:\F\to\G$ is given by
\begin{align*}
    S_P^*[f](a,b) = \int_{\RR^m} f(x) \overline{\sigma(a \cdot x - b)} \dd P(x).
\end{align*}
If $P$ is obvious from the context, we write $S_P^*$ as $S^*$ for simplicity.
\end{prop}
\begin{proof}
We can verify this by the direct calculation: For any $f \in \F$ and $\mud \in \G$,
\begin{align*}
    \iprod{ f, S[\mud] }_{\F}
    = \int_{\RR^m \times \RR \times \RR^m} f(x) \overline{\sigma(a \cdot x - b )} \overline{\phi(a,b)} \dd a \dd b \dd P(x)
    = \iprod{S_P^*[f],\mud}_{\G},
\end{align*}
as long as one of the integrals exists. 
\end{proof}

%\textbf{Ridgelet transform $R$.}
\begin{dfn}[Ridgelet transform $R$]
For any measure $P$ on $\RR^m$ and function $\rho:\RR\to\CC$,
we define the \emph{ridgelet transform} of $f$ on $\RR^m$ by
\begin{align*}
    R_P[f;\rho](a,b)
    &:= \int_{\RR^m} f(x) \overline{\rho( a \cdot x - b )} \dd P(x),  (a,b) \in \RR^m \times \RR.
\end{align*}
If $P$ and/or $\rho$ are obvious from the context, we write $R_P[f;\rho]$ as $R[f]$ for simplicity. In addition, when we emphasize the Lebesgue measure case $P=\dd x$, we write $R_{\dd x}$.
\end{dfn}
In particular, the adjoint operator $S_P^*$ is a ridgelet transform: $S_P^*[f] = R_P[f;\sigma]$.
%(We need the Lebesgue measure case $P = \dd x$ for the reconstruction formula.)

%\textbf{Fourier expression of $R$.} %For $f \in L^2(P)$ and $\rho \in L^2(\RR)$,
\begin{prop}[Fourier expression of $R$]
\begin{align*}
    R_P[f;\rho](a,b) = \frac{1}{2 \pi} \int_{\RR} \widehat{f \dd P}(\omega a) \overline{\rho^{\sharp}(\omega)} e^{i \omega b} \dd \omega,
\end{align*}
where $\widehat{f \dd P}$ denotes the Fourier transform of the measure $f \dd P$. When $P=\dd x$, then $\widehat{f \dd x}$ is naturally identified with the ordinary Fourier transform $\widehat{f}$.
\end{prop}
\begin{proof}
Since $\rho( a \cdot x - b ) = \frac{1}{2 \pi} \int_\RR \rho^{\sharp}(\omega) e^{i \omega (a \cdot x - b)} \dd \omega$,
\begin{align*}
    R_P[f;\rho](a,b)
    &= \frac{1}{2 \pi} \int_{\RR^m} f(x) \int_{\RR} \overline{\rho^{\sharp}(\omega)} e^{-i \omega (a \cdot x - b)} \dd \omega \dd P(x) \\
    &= \frac{1}{2 \pi} \int_{\RR} \widehat{f\dd P}(\omega a) \overline{\rho^{\sharp}(\omega)} e^{i \omega b} \dd \omega. \qedhere
\end{align*}
\end{proof}

%\textbf{{Admissibility} condition.}
%We say {that} two functions $\rho,\sigma:\RR\to\CC$ are \emph{admissible} if these satisfy
%\begin{align*}
%    (2 \pi)^{m-1} \int_{\RR} \frac{\sigma^\sharp(\omega)\overline{\rho^{\sharp}(\omega)}}{|\omega|^m} \dd \omega = 1.
%\end{align*}
%In particular, when $\rho=\sigma$, we say $\rho$ (or $\sigma$) is \emph{admissible with itself}, {or self-admissible}.
% We say functions $\rho \in L^2(\RR)$ is \emph{admissible} if it satisfies
% \begin{align}
%     (2 \pi)^{m-1} \int_{\RR} \frac{|\rho^{\sharp}(\omega)|^2}{|\omega|^m} \dd \omega = 1.
% \end{align}

We remark that satisfying this admissible condition is not difficult. For example, take a Gaussian $\rho_0(t) := \exp(-t^2/2)$, and put $\rho_n(t) := C \rho_0^{(n)}(t)$ with an integer $n$ such that $2n-m > 0$ and a positive constant $C$. Then, by appropriately setting $C$, $\rho_n$ can be admissible (with itself) because $(2 \pi)^{m-1} \int_{\RR} |\rho_n^{\sharp}(\omega)|^2/|\omega|^{m} \dd \omega = (2 \pi)^{m-1} C \int_\RR |\omega|^{2n-m} |\rho_0^\sharp(\omega)|^2 \dd \omega < \infty$ and we can set $C$ for the integral to be normalized as $1$.

%\textbf{Reconstruction formula.}
% Obviously, non-integrable activation functions $\sigma$ such as Tanh and ReLU cannot be admissible. In the following, we show that the reconstruction is still possible when $\sigma$ is an ``anti-derivative'' of some admissible function $\rho$. For example, ReLU $\sigma$ is an anti-derivative of the $\rho_n$ (in the point-wise limit), namely, $\sigma^{(2 + n)}(t) = \lim_{\eps \to 0} \rho_n(t/\eps)/\eps$.
\begin{prop}[Reconstruction formula]
%Assume that for the activation function $\sigma$ of $S$, there exists an admissible function $\rho \in L^2(\RR)$ such that $\sigma^{(k)} = \rho$. Then, $S[S^*_P[f]] = f \dd P$ for any $f \in L^2(P)$.
For any $f \in L^1(P)$, $S[R_P[f;\rho]] = f \dd P$.
\end{prop}
% \begin{proof} We begin with the case $P=\dd x, f \in L^1(\RR^m)$ and $\widehat{f} \in L^1(\RR^m)$.
% \begin{align}
%     S[S_{\dd x}^*[f]](x)
%     &= \int_{\RR^m \times \RR} \partial_b^{2k} \left[ \int_{\RR^m} f(y) \overline{\sigma(a \cdot y - b )} \dd y \right] \sigma(a \cdot x - b ) \dd a \dd b \\
% %    &= \int_{\RR^m \times \RR} \left[ \int_{\RR^m} f(y) \overline{\rho(a \cdot y - b )} \dd x \right] \rho(a \cdot x - b ) \dd a \dd b \\
%     &= \int_{\RR^m} \int_{\RR^m} f(y) (\overline{\widetilde{\rho}} * \rho) (a \cdot (x-y)  ) \dd y \dd a\\
%     &= \frac{1}{2 \pi} \int_{\RR^m} \int_{\RR^m} f(y) \int_\RR |\rho^{\sharp}(\omega)|^2 e^{i \omega a \cdot (x-y)} \dd \omega \dd y \dd a \\
%     &= \frac{1}{2 \pi} \int_{\RR^m} \int_{\RR^m} f(y) \left[ \int_\RR \frac{|\rho^{\sharp}(\omega)|^2}{|\omega|^m} \dd \omega \right] e^{i a \cdot (x-y)} \dd y \dd a \\
%     &= \frac{1}{(2 \pi)^m} \int_{\RR^m} \widehat{f \dd P}(a) e^{i a \cdot x} \dd a = f(x) \dd P(x).
% \end{align}
% \end{proof}
\begin{proof} We write $f_P := f \dd P$ for short.
%We use the Fourier expression: $S^*_P[f](a,b) = (2\pi)^{-1} \int_{\RR} \widehat{f_P}(\omega a) \overline{\rho^\sharp(\omega)}e^{i\omega b} \dd \omega$. Then,
By using the Fourier expressions, we have
\begin{align*}
    S[R_P[f;\rho]](x)
    % &= \frac{1}{2\pi} \int_{\RR^m \times \RR}\left[ \int_{\RR} \widehat{f_P}(\omega a) \overline{\rho^\sharp(\omega)} e^{i\omega b} \dd \omega \right] \sigma(a \cdot x - b ) \dd a \dd b \\
    % &= \frac{1}{2\pi} \int_{\RR^m \times \RR} \widehat{f_P}(\omega a) \left[ \overline{\sigma^\sharp(\omega)} \int_{\RR} e^{i\omega (a\cdot x-b)} \sigma(b)  \dd b \right] \dd \omega \dd a \\
    &= \frac{1}{2\pi} \int_{\RR^m \times \RR} \widehat{f_P}(\omega a) \sigma^\sharp(\omega) \overline{\rho^\sharp(\omega)} e^{i\omega a \cdot x} \dd \omega \dd a \\
%    &= \int_{\RR^m \times \RR} \left[ \int_{\RR^m} f(y) \overline{\rho(a \cdot y - b )} \dd x \right] \rho(a \cdot x - b ) \dd a \dd b \\
%    &= \int_{\RR^m} \int_{\RR^m} f(y) (\overline{\widetilde{\rho}} * \rho) (a \cdot (x-y)  ) \dd y \dd a\\
%    &= \frac{1}{2 \pi} \int_{\RR^m} \int_{\RR^m} f(y) \int_\RR |\rho^{\sharp}(\omega)|^2 e^{i \omega a \cdot (x-y)} \dd \omega \dd y \dd a \\
    &= \frac{1}{2 \pi} \int_{\RR^m} \widehat{f_P}(\xi) \left[ \int_\RR \frac{\sigma^\sharp(\omega) \overline{\rho^\sharp(\omega)}}{|\omega|^m} \dd \omega \right] e^{i \xi \cdot x} \dd \xi \\
    %&= \frac{1}{(2 \pi)^m} \int_{\RR^m} \widehat{f \dd P}(a) e^{i a \cdot x} \dd a
    &= f_P(x).
\end{align*}
Here, we change the variable $(a,\omega) = (\xi/\omega,\omega)$ with $\dd a\dd \omega = |\omega|^{-m}\dd\xi\dd\omega$ in the second equation.
\end{proof}
We remark that when $\sigma$ is self-admissible,
the reconstruction formula can be extended to $f \in L^2(\RR^m)$ by using the Plancherel formula below.

%\textbf{Plancherel formula.}
The following isometries play an important role in the proof of main results as we can regard $|R[f](a,b)|^2$ with a ``density function'' of the parameter distribution.
\begin{prop}[Plancherel formula]\phantom{a}\\
\begin{itemize}
    \item When $\sigma = \rho$, $\| S_P^*[f] \|_{\G}^2 = \iprod{f, f \dd P}_{\F}$ because
    \begin{align*}
    \| S_P^*[f] \|_{\G}^2\!
    = \!\iprod{ S_P^*[f], S_P^*[f] }_{\G}\!
    = \!\iprod{ f, S[S_P^*[f]] }_\F\!
    = \!\iprod{ f, f \dd P }_\F.
%    \red{\red \le \| f \|_\F \| p \|_\infty}
    \end{align*}
    \item When $\sigma = \rho$, $\| f \|_{\F}^2=\iprod{S_P^*[f], S_{\dd x}^*[f]}_{\G}$ because
    \begin{align*}
        &\| f \|_{\F}^2
        =\iprod{f, f}_{\F}
        =\iprod{f, S[S_{\dd x}^*[f]]}_{\F}
        =\iprod{S_P^*[f], S_{\dd x}^*[f]}_{\G}.
%        \red{ \le \| S_P^*[f] \|_{\G} \| S_{\dd x}^*[f] \|_{\G}}
    \end{align*}
%     \item When $\sigma$ and $\rho$ are admissible,
%     %let $\tau$ be a solution of $|\tau^\sharp|^2 = \sigma^\sharp \overline{\rho^\sharp}$. Then,
%     \begin{align}
%         &\| f \|_{\F}^2
%         =\iprod{f, f}_{\F}
%         =\iprod{f, S[R_{\dd x}[f;\rho]]}_{\F}
%         =\iprod{S_P^*[f], R_{\dd x}[f;\rho]}_{\G}.
% %        \red{ \le \| S_P^*[f] \|_{\G} \| S_{\dd x}^*[f] \|_{\G}}
%     \end{align}
     \item When $f$ is supported in a set $\X\subset\RR^m$ and $P=1_\X \dd x$ (indicator function), then $f\dd P = f$, and thus $S_{\dd x}^*[f] = S_P^*[f]$, and the above two identities coincide:
    $\| S_P^*[f] \|_{\G}^2 = \| f \|_{\F}^2$.
    %, and $\| f \|_{\F}^2 = \iprod{ R_P[f;\rho], R_P[f;\sigma] }_{\G} = \| R_P[f;\tau] \|_{\G}$ where $|\tau^\sharp|^2 = \sigma^\sharp \overline{\rho^\sharp}$.
    % \item Let $\phi = R_P[f;\rho]$ and thus $S[\phi] = f_P$, then $\iprod{ R_P[f;\sigma], \phi }_\G = \| S[\phi] \|_\F^2$.
    %  \item Some Fourier expressions:
    %  \begin{align}
    %     \| R_P[f;\rho] \|_\G^2
    %     &= \int_{\RR^m} \left[\int_\RR \Bigg| \frac{1}{2\pi} \int_\RR \widehat{f_P}(\omega a) \overline{\rho^\sharp(\omega)} e^{i\omega b} \dd \omega \Bigg|^2 \dd b \right]\dd a \\
    %     &= \int_{\RR^m} \left[\int_\RR \big| \widehat{f_P}(\omega a) \overline{\rho^\sharp(\omega)} \big|^2 \dd \omega \right]\dd a
    %     \quad \left( = \| \widehat{f_P} \|_{L^2(\RR^m)}  \right)\\
    %     \iprod{ R_P[f;\rho], R_P[f;\sigma] }_\G
    %     &= \int_{\RR^m} \left[\int_\RR |\widehat{f_P}(\omega a)|^2 \overline{\rho^\sharp(\omega)} \sigma^\sharp(\omega) \dd \omega \right]\dd a
    %     \quad \left( = \| \widehat{f_P} \|_{L^2(\RR^m)}  \right).
    %  \end{align}
    %  \red{The last equation is interesting because as long as $\rho$ and $\sigma$ are admissible, then the $L^2$-norm is conserved.}
\end{itemize}
\end{prop}
% \begin{thm}
%     Assume that $\rho$ is admissible, and $\sigma = \rho^{(k)}$ for an integer $k$. Then,
%     $\| S_{k,\red{\dd x}}^*[f] \|_{\G}^2 = \| R_{\dd x}[f;\rho] \|_\G^2 = \| f \|_\F^2.$
% \end{thm}
% \begin{proof}
% By the definition of the adjoint operator and the reconsruction formula, we have
% \begin{align}
% \| S^*[f] \|_{\G}^2 = \iprod{ S^*[f], S^*[f] }_{\G} = \iprod{ f, S[S^*[f]] }_\F = \| f \|_\F^2.
% \end{align}
% On the other hand,
% \begin{align}
% \iprod{ S_P^*[f], S_P^*[g] }_{\G} = \iprod{ R_P[f;\sigma^{(k)}], R_P[g;\sigma^{(k)}] }_{\G}^2
% \end{align}
% for any f,g \in \Ff,g \in \F.
% Hence, we have
% \begin{align}
% \| S^*[f] \|_{\G}^2 = \iprod{ S^*[f], S^*[f] }_{\G} = \iprod{ R[f], R[f] }_{\G}^2 = \| R[f] \|_{\G}^2,
% \end{align}
% which yields the conclusion.
% % \begin{align}
% %     \| S^*[f] \|_{\G}^2
% %     &= \iprod{ S^*[f], S^*[f] }_{\G} = \iprod{ f, S[S^*[f]] }_\F = \| f \|_\F^2 \\
% %     &= \iprod{ R[f], R[f] }_{\G}^2 = \| R[f] \|_\G^2. %\int_{\RR^m \times \RR} \Bigg| \int_{\RR^m} f(x) \rho(a \cdot x - b) \dd x \Bigg|^2 \dd a \dd b.
% % \end{align}
% \end{proof}

\section{Proofs for Theorems}\label{supp:proofs}
\subsection{\refthm{minimizer}} \label{supp:proof.minimizer}
%\subsection{Theorem \ref{thm:minimizer}\ref{thm:minimizer}} \label{supp:proof.minimizer}
We impose assumptions as below.
\begin{itemize}
    \item[(A1)] Let $\Omega$ be a bounded open subset with smooth boundary in the input domain $\RR^m$, and put $K := \overline{\Omega}$. Namely, $K$ is a compact set. The boundedness assumption is required for the loss $\| f - g_d \|_{L^2(K)}$ between $f$ and \emph{finite net} $g_\num(x) = \sum_{i=1}^\num c_i \sigma(a_i\cdot x - b_i)$ exists finite. We note that  $\| g_\num \|_{L^2(\RR^m)} = \infty$ simply because $\sigma(a \cdot x - b)$ has a constant direction, while $\| g_\num \|_{L^2(K)} < \infty$. 
    The closure assumption excludes degenerated cases such as $K=\{ x_1, \ldots, x_n \}$ (isolated points) for the sake of simplicity. The smooth boundary is required in \refthm{ubound}, to continuously embed $H^s(\Omega)$ to $H^s(\RR^m)$ via zero-extension.
    \item[(A2)] Let $f : \RR^m \to \CC$ be an square-integrable function supported in the compact set $K$. Namely, $f \in L^2(K)$. Both integrability and compact-support assumptions exclude the so-called ``teacher-student setting'' where $f$ is a finite neural network such as $\sum_{i=1}^d c_i \sigma(a_i \cdot x - b_i)$. \label{assm:f}
    \item[(A3)] $P := 1_K \dd x$ (i.e., the volume is not normalized to $1$), which yields $\F = L^2(P) = L^2(K)$ and $S^*[f]=R[f;\sigma,1_K\dd x] = S_K^*[f]$.
    \item[(A4)] $C_{\sigma,P}$ exists finite, namely $\| S[\mu] \|_\F \le C_{\sigma,P} \| \mu \|_{TV}$ so that $S(\M) \subset \F$.
    \item[(A5)] Let $\sigma:\RR\to\CC$ be a measurable function that is admissible with itself.
%    \item V = [-\lambda,\lambda]^m \times [-\kappa/2,\kappa/2]V = [-\lambda,\lambda]^m \times [-\kappa/2,\kappa/2]
\end{itemize}
Then, the approximation error is lower bounded by the volume of the tail part (i.e., outside the parameter domain $V$) of the ridgelet spectrum:
\begin{align*}
    \inf_{\mu \in \M(V)} \| f - S[\mu] \|_{L^2(K)}^2
%    &= \|f\|_{L^2(K)}^2 - \| S[\mu^*] \|_{L^2(K)}^2 \\
    \ge \|f\|_{L^2(K)}^2 - \| S[ S_K^*[f]|_V ] \|_{L^2(K)}^2
%    &= \| S[ S^*[f]|_{V^c} ] \|_{L^2(\RR^m)}^2 + \| \phi_0 \|_{L^2(\RR^m\times\RR)}^2
    \ge
        \|f\|_{L^2(K)}^2 - \| S_K^*[f] \|_{L^2(V)}^2
    = \| S_K^*[f] \|_{L^2(V^c)}^2.
\end{align*}
Obviously, the lower bound is strictly positive when the tail density $S^*[f]|_{V^c}$ is positive. 

%\subsubsection{Proof 3-0}
\begin{proof}
%We write $\| \cdot \|_{L^2(\RR^m\times\RR)}$ as $\| \cdot \|$ for short.
We write $\G := L^2(\RR^m\times\RR)$ for short.
%We fix a function $f \in L^2(\RR^m)$ that is supported on a compact set $K \subset \RR^m$.
By $S^*[f]|_V$ (resp. $S^*[f]|_{V^c} = S^*[f] - S^*[f]|_V$) we write the truncation of the ridgelet spectum $S^*[f]$ onto $V$ (resp. $V^c$). By $\proj_{\ker S}$ (resp. $\proj_{(\ker S)^\perp}$) we write the projection from $\M(V)$ to the null space $\ker S$ (resp. to its complement $(\ker S)^\perp$).

\textbf{Step~1.}
Let $\mu^*$ denote an arbitrary single element of the minimizers in $\M(V)$.
We note that $\mu^*$ always exists as a consequence of the following \emph{extreme value theorem}: 
\begin{prop}
Suppose $E$ be a Banach space, $X$ be a closed convex subset of $E$, and $\varphi:X\to(\infty,\infty]$ be a coercive lower semi-continuous function. (Here, coercive means $\varphi(x) \to +\infty$ as $\| x \|_E \to \infty$.) Then, there exists an element (minimizer) $x^* \in X$ that attains the minimum, i.e., $\inf_{x \in X} \varphi(x) = \varphi(x^*)$.
\end{prop}
Now $E=X=\M(V)$ is a Banach space (known as an \emph{rca space}), which means it is closed and convex, and $\varphi(\mu) := \| f - S[\mu] \|_{L^2(K)}^2$ is coercive and Lipschitz continuous, there exists a minimizer $\mu^* \in \M(V)$.
%Particularly, by assumptions (A2) and (A4), $f \in L^1(\RR^m)$, we may put $\mu^* = R[f]$. However, we cannot verify $\supp \mu \subset V$.
Namely, we have
\begin{align}
    \inf_{\mu \in \M(V)} \| f - S[\mu] \|_{L^2(K)}^2
    &= \| f - S[\mu^*] \|_{L^2(K)}^2. \label{eq:minimizer}
\end{align}

%\textbf{Step~2.}
Since the minimizer $S[\mu^*]$ satisfies the Pythagorean relation:
\begin{align}
    \| S[\mu^*] \|_{L^2(K)}^2 + \| f - S[\mu^*] \|_{L^2(K)}^2 = \| f \|_{L^2(K)}^2, \label{eq:pythagorean}
\end{align}
we have
\begin{align}
%\| f - S[\mu^*] \|_{L^2(K)}^2
\eqref{eq:minimizer}
     = \| f \|_{L^2(K)}^2 - \| S[\mu^*] \|_{L^2(K)}^2. \label{eq:minimizer2}
\end{align}

\textbf{Step~2.}
We show the following inequality:
\begin{align}
    \| S[\mu^*] \|_{L^2(K)} \le \| S[ S^*[f]|_V ] \|_{L^2(K)}, \label{eq:key.ineq1}
\end{align}
which yields the following lower bound:
\begin{align}
%    \inf_{\mu \in \M(V)} \| f - S[\mu] \|_{L^2(K)}^2
%    &= \| f - S[\mu^*] \|_{L^2(K)}^2 \notag\\
%    &= \|f\|_{L^2(K)}^2 - \| S[\mu^*] \|_{L^2(K)}^2 \notag\\
\eqref{eq:minimizer2}
    &\ge \|f\|_{L^2(K)}^2 - \| S[ S^*[f]|_V ] \|_{L^2(K)}^2 \label{eq:proof.lb1}
%    &= \| S[ S^*[f]|_{V^c} ] \|_{L^2(\RR^m)}^2 + \| \phi_0 \|_{L^2(\RR^m\times\RR)}^2
\end{align}
\begin{proof}[Proof of \emph{(\ref{eq:key.ineq1})}]
To estimate the norm of $S[\mu^*]$, we can neglect the null component of $\mu^*$, say $\mu_0^* \in \ker (S:\M(V)\to L^2(K))$, since it satisfies
\begin{align}
    \|S[\mu^*]\|_{L^2(K)} = \|S[\mu^* - \mu_0^*]\|_{L^2(K)}, \quad \mbox{and} \quad
    \iprod{f, S[\mu^*]}_{L^2(K)} = \iprod{S_K^*[f], \mu^*-\mu_0^*}_{L^2(\RR^m\times\RR)},
\end{align}
for any $f \in L^2(K)$.
The Pythagorean relation \eqref{eq:pythagorean} is rephrased as
\begin{align}
     \| S[\mu^*] \|_{L^2(K)}^2
     &= \Re \iprod{ f, S[\mu^*] }_{L^2(K)}. \label{eq:norm.iprod1}
%     \intertext{Since $f$ and $\mu^*$ are supported in $K$ and $V$ respectively,}
     \intertext{Since $\mu^*$ is supported in $V$,}
%$     &= \int_{\RR^m \times (\RR^m \times \RR)} f(x) \sigma(a\cdot x - b) \dd x \dd \mu^*(a,b)\\
%    &= \iprod{ f, S[ \mu^* ]|_K }_{L^2(\RR^m)}\\
    &= \Re \iprod{ S^*_K[f] |_V, \mu^* }_{L^2(\RR^m\times\RR)}.\\
    \intertext{Since $\mu^*$ is assumed not to contain the null component,}
    &= \Re \iprod{ \proj_{(\ker S)^\perp}[S^*_K[f] |_V], S^*_K[ S[\mu^*]] }_{L^2(\RR^m\times\RR)},\\
%    &= \iprod{ S^*[f], S^*[ S[ \mu^* ]|_K ] }_{L^2(\RR^m\times\RR)}\\
%    &= \iprod{ S^*[f]|_V, S^*[ S[ \mu^* ]|_K ] }_{L^2(\RR^m\times\RR)}
    \intertext{By the definition of $S_K^*$,}
%    &=\iprod{ S[ S^*[f]|_V ], S[ \mu^* ]|_K }_{L^2(\RR^m)}\\
    &=\Re \iprod{ S[ S^*[f]|_V ], S[ \mu^* ] }_{L^2(K)}.
    \intertext{By the Cauchy-Schwartz inequality,}
    &\le \| S[ S^*[f]|_V ] \|_{L^2(K)} \|S[ \mu^* ] \|_{L^2(K)},
%%%%%%
%     \intertext{Since $f$ and $\mu^*$ are supported in $K$ and $V$ respectively,}
%     &= \int_{\RR^m \times (\RR^m \times \RR)} f(x) \sigma(a\cdot x - b) \dd x \dd \mu^*(a,b)\\
%     \intertext{Since $f$ is supported in $K$,}
%    &= \iprod{ f, S[ \mu^* ] }_{L^2(\RR^m)}.
%    \intertext{By the Plancherel relations,}
%    &= \iprod{ S^*[f], S^*[ S[ \mu^* ]|_K ] }_{L^2(\RR^m\times\RR)}.\\
%    \intertext{By changing the order of integration,}
%    &= \iprod{ S^*[f], \mu^*}_{L^2(\RR^m\times\RR)}.
    % \intertext{Since $\mu$ is supported in $V$,}
%    &= \iprod{ S^*[f]|_V, \mu^*}_{L^2(\RR^m\times\RR)}
%    \intertext{By the Plancherel relations,}
%    &=\iprod{ S[ S^*[f]|_V ], S[ \mu^* ] }_{L^2(\RR^m)}
%    &=\iprod{ f, S[ \mu^* ] }_{L^2(K)} - \iprod{ S[ S^*[f]|_{V^c} ], S[ \mu^* ] }_{L^2(\RR^m)}.
%%%%%%
%    &\iprod{ S[ S^*[f]|_V ], S[ \mu^* ] }_{L^2(K)}
%    = \iprod{ S^*[f]|_V, S^*[ S[ \mu^* ]|_K ] }_{L^2(\RR^m\times\RR)}\\
%    &= \iprod{ S^*[f], S^*[ S[ \mu^* ]|_K ] }_{L^2(\RR^m\times\RR)}
%    = \iprod{ f, S[ \mu^* ]|_K }_{L^2(\RR^m)} 
%    = \iprod{ f, S[ \mu^* ] }_{L^2(K)} 
%    = \| S[\mu^*] \|_{L^2(K)}^2,
%%%%%%
\end{align}
which yields the inequality (\ref{eq:key.ineq1}).
\end{proof}

\textbf{Step~3.}
Next, we show the following equalities:
\begin{align}
    \| f \|_{L^2(K)}^2
    &= \| S_K^*[f] \|_{L^2(V)}^2 + \| S_K^*[f] \|_{L^2(V^c)}^2 \notag \\
    &= \| S[ S_K^*[f]|_V ] \|_{L^2(K)}^2 + \| S[S_K^*[f]|_{V^c}] \|_{L^2(K)}^2+ 2 \| \phi_0 \|_{L^2(\RR^m\times\RR)}^2, \label{eq:key.eq1}
\end{align}
where $\phi_0$ is the null component of $S_K^*[f]|_V$ defined later,
and this equality refines the lower bound as
\begin{align*}
    (\ref{eq:proof.lb1})
    &= \| S[ S_K^*[f]|_{V^c} ] \|_{L^2(K)}^2 + 2\| \phi_0 \|_{L^2(\RR^m\times\RR)}^2
%    &= \| S^*[f]|_{V^c} \|_{\G}^2 + \| \phi_0 \|_{L^2(\RR^m\times\RR)}^2 \\
    \ge \| S^*[f] \|_{L^2(V^c)}^2 = \| f \|_{L^2(K)}^2 - \| S^*[f] \|_{L^2(V)}^2.
\end{align*}
\emph{Proof of}(\ref{eq:key.eq1}). By using the Plancherel formula and splitting the integral, the restrictions $S^*[f]|_V$ and $S^*[f]|_{V^c} (=S^*[f] - S^*[f]|_V)$ of ridgelet spectra $S^*[f]$ satisfy the following equation:
\begin{align}
    \| f \|_{L^2(K)}^2
%    &= \| f \|_{L^2(\RR^m)}^2
%    &= \iprod{ f, S[S_K^*[f]] }_{L^2(\RR^m\times\RR)} 
    &= \| S_K^*[f] \|_{L^2(\RR^m\times\RR)}^2 \nonumber\\
    &= \left( \int_V + \int_{V^c} \right) |S_K^*[f](a,b)|^2 \dd a \dd b = \| S_K^*[f] \|_{L^2(V)}^2 + \| S_K^*[f] \|_{L^2(V^c)}^2. \label{eq:key.eq2}
%    &= \| S[S^*[f]|_V] \|_{L^2(\RR^m)}^2 + \| S[S^*[f]|_{V^c}] \|_{L^2(\RR^m)}^2 + 2 \| \phi_0 \|_{\G}.
\end{align}

%%%%%%%%%%%%%%%%%%%%%%%%%%%%%%%%%%%%%%%%%%%%%%%

%By the Plancherel formula,
% \begin{align}
%     \| f \|_{\F}^2
%     = \| S^*[f] \|^2
%     = \|S^*[f]|_V + S^*[f]|_{V^c}\|^2 \label{eq:divided.integral}.
% \end{align}
In order to further decompose the equation (\ref{eq:key.eq2}), we consider the null components of the restrictions $S_K^*[f]|_V$ and $S_K^*[f]|_{V^c}$.
Recall that the operator $S:{L^2(\RR^m\times\RR)}\to{L^2(K)}$ has a non-trivial null space $\ker (S:{L^2(\RR^m\times\RR)}\to{L^2(K)})$, and its orthogonal complement is given by the image space $\im S_K^*$ of the adjoint operator $S_K^*:{L^2(K)}\to{L^2(\RR^m\times\RR)}$, namely, $(\ker S)^\perp = \im S_K^*$.
Hence, the entire space $\G := L^2(\RR^m\times\RR)$ is decomposed into the orthogonal direct sum: $\G = \ker S \oplus \im S_K^*$.
By definition, $S_K^*[f] \in \im S_K^* = (\ker S)^\perp$. Nevertheless, its restrictions $S_K^*[f]|_V$ and $S_K^*[f]|_{V^c}$ may have null components.
We write the (potentially non-trivial) null component of $S_K^*$ as $\phi_0 := \proj_{\ker S}[ S_K^*[f]|_V ]$, and its orthogonal component as $\phi_V := \proj_{(\ker S)^\perp}[ S_K^*[f]|_V ]$, so that both components become a direct sum: $\phi_V \oplus \phi_0 = S_K^*[f]|_V$.
%We show that the null component is trivial: \phi_0 = 0\phi_0 = 0.
Then, the null component of $S_K^*[f]|_{V^c}$ is $-\phi_0$ because the sum $S_K^*[f]|_V + S_K^*[f]|_{V^c} = S_K^*[f]$ is in the image space $\im S_K^*$,
and thus the orthogonal component $\phi_{V^c} := \proj_{(\ker S)^\perp}[ S_K^*[f]|_{V^c} ]$ is given by $\phi_{V^c} = S_K^*[f]|_{V^c} + \phi_0$.
Hence, by using the orthogonality and the Plancherel formula, the equation (\ref{eq:key.eq2}) is further calculated as follows:
\begin{align*}
    (\ref{eq:key.eq2})
    &= \| \phi_V \|_{L^2(\RR^m\times\RR)}^2 + \| \phi_{V^c} \|_{L^2(\RR^m\times\RR)}^2 + 2 \| \phi_0 \|_{L^2(\RR^m\times\RR)}^2 \\
    &= \| S[ S_K^*[f]|_V ] \|_{L^2(K)}^2+ \| S[S_K^*[f]|_{V^c}] \|_{L^2(K)}^2 + 2 \| \phi_0 \|_{L^2(\RR^m\times\RR)}^2.
\end{align*}
%\end{proof}
Combining Steps 1, 2, and 3, we have the assertion.
% %%%%%%%%%%%%%%%%%%%%%%%%%%%%%%%%%%%%%%%%%%%%%%%
\end{proof}

\subsection{\refthm{ubound}}\label{supp:proof.ubound}
%\subsection{Theorem \ref{thm:ubound}}\label{supp:proof.ubound}
We write $\iprod{x} := (1+|x|^2)^{1/2}$ for $x \in \RR^m$, which satisfies $\max\{1,|x|\} \le \iprod{x}$ for any $x$. 
For square-integrable functions $f \in L^2(\RR^m)$ on whole space $\RR^m$, we employ $\| f \|_{H^s(\RR^m)}^2 := \int_{\RR^m} |\widehat{f}(\xi)|^2 (1+|\xi|^2)^s \dd \xi$ for the $L^2$-Sobolev norm of order  $s \in \RR$.
For functions on an open subset $\Omega$ with $C^1$-boundary, we define the $L^2$-Sobolev space $H^s(\Omega)$ with $s \in (1/2,\infty]$ by continuously embedding it to $H^s(\RR^m)$. Namely, we identify $f \in H^s(\Omega)$ with $\overline{f} \in H^s(\RR^m)$ that is compactly supported in $\Omega$ and satisfies $\overline{f}|_\Omega = f$.
%Hence, the $L^2$-Sobolev space $H^2(\RR^m)$ is a (closed) subspace of $L^2$-functions $f \in L^2(\RR^m)$ that have finite Sobolev norm: $\| f \|_{H^s} < \infty$.

\textbf{Decay property.}
Suppose that $\rho$ is self-admissible, namely, $\int_{\RR} |\rho^\sharp(\omega)|^2|\omega|^{-m}\dd \omega = (2\pi)^{m-1}$.
For any $f \in H^s(\Omega)$,
\begin{align*}
    | R[f](ru,b) |
    &\le \frac{1}{2\pi} \int_{\RR} | \widehat{f}(\omega u) | |\rho^{\sharp}(\omega/r)/r| \dd \omega\\
    &\le \frac{1}{2\pi} \int_{\RR} \left( |\omega u|^s |\widehat{f}(\omega u) \omega^{\frac{m-1}{2}}| \right) \left( |\omega^{-(2s+m-1)/2} \rho^{\sharp}(\omega/r)/r| \right) \dd \omega \\
    &\le \frac{1}{2\pi} \left(2 \int_0^\infty \iprod{\omega u}^{2s} |\widehat{f}(\omega u)|^2 \omega^{m-1} \dd \omega\right)^{1/2} \left( \!|r|^{-2s-m}\!\! \int_{\RR}  \iprod{\omega}^{-2s+1} |\rho^{\sharp}(\omega)|^2 |\omega|^{-m} \dd \omega \!\!\right)^{\frac{1}{2}} \\
    &= C_{\rho,s} |r|^{-(2s+m)/2} \Phi_s[f](u).
\end{align*}
Here, we write $\Phi_s[f](u) := \left(\int_0^\infty \iprod{\omega u}^{2s} |\widehat{f}(\omega u)|^2 \omega^{m-1} \dd \omega\right)^{1/2}$ for future use, of which the spherical mean becomes the Sobolev norm:
\begin{align*}
    \int_{\Sph^{m-1}} \Phi_s[f](u)^2 \dd u = \| f \|_{H^s}^2;
\end{align*}
and the constant $C_{\rho,s}$ is given and bounded as
\begin{align*}
C_{\rho,s}^2
&:= \frac{2}{(2 \pi)^2} \int_{\RR} \iprod{\omega}^{-2s+1} \frac{|\rho^{\sharp}(\omega)|^2}{|\omega|^m} \dd \omega\le \frac{2}{(2 \pi)^2} \int_{\RR} \frac{|\rho^{\sharp}(\omega)|^2}{|\omega|^m} \dd \omega = 2(2\pi)^{m-3} < \infty,
\end{align*}
because $\iprod{\omega}^{-2s+1} \le 1$ as long as $-2s+1 \le 1$.
%\textbf{Spherical average.}
% In particular, the spherical mean is bounded by the Sobolev norm:
% \begin{align}
%     \int_{\Sph^{m-1}} | R[f](ru,b) |^2 \dd u
%     &\le C_{\rho,s}^2 \| f \|_{H^s}^2 r^{-2s-m}. \label{eq:decay.sobolev} %\int_\eps^\delta r^{-2s-1} \dd r.
% \end{align}
% In particular,
% \begin{align}
%     \int_{\Sph^{m-1}} \int_\eps^\delta | R[f](ru,b) |^2 r^{m-1} \dd r \dd u
%     &\le C_{\rho,s}^2 \| f \|_{H^s}^2 \frac{1}{2s}\left( \frac{1}{\eps^{2s}} - \frac{1}{\delta^{2s}} \right). \label{eq:decay.sobolev} %\int_\eps^\delta r^{-2s-1} \dd r.
% \end{align}

\textbf{Auxiliary estimates.} The obtained estimate does not depend on $b$ and diverges at $r=0$, but $R[f]$ usually depends on $b$ and does not always diverge at $r=0$. Hence, we derive auxiliary estimates.
% We begin with fixing a \in \RR^ma \in \RR^m.
% Recall that the Fourier expression of R[f](a,b)R[f](a,b) %=\frac{1}{2\pi}\int_\RR \widehat{f}(a\omega)\rho^\sharp(\omega)e^{i\omega b}\dd \omega=\frac{1}{2\pi}\int_\RR \widehat{f}(a\omega)\rho^\sharp(\omega)e^{i\omega b}\dd \omega
% implies that R[f](a,\cdot)R[f](a,\cdot) is the Fourier transform of \widehat{f}(\cdot a) \rho^\sharp(\cdot)\widehat{f}(\cdot a) \rho^\sharp(\cdot) at every aa.
% Since f \in H^s(\RR^m)f \in H^s(\RR^m), \widehat{f}(\cdot a) \rho^\sharp(\cdot) \in L^1(\RR)\widehat{f}(\cdot a) \rho^\sharp(\cdot) \in L^1(\RR) at every aa.
% But this implies R[f](a,\cdot) \in BUC(\RR)R[f](a,\cdot) \in BUC(\RR) (bounded uniform continuous) at every fixed aa by the Riemann-Lebesgue theorem.
% On the other hand, f \in H^s(\RR^m)f \in H^s(\RR^m) also implies R[f](\cdot,b) \in H^s(\RR^m)R[f](\cdot,b) \in H^s(\RR^m) at every fixed bb.
% But this implies that R[f] \in BUC(\RR^m\times\RR)R[f] \in BUC(\RR^m\times\RR). %and so the constant C_{f,\rho}C_{f,\rho}
By the assumption that $f \in L^1(K)$ and $\rho \in L^\infty(\RR)$, we have
\begin{align*}
    | R[f](a,b) |
    &\le \int_{\RR^m} |f(x)| |\rho(a \cdot x - b)| \dd x \le \| f \|_{L^1(K)} \| \rho(a \cdot x - b) \|_{L^\infty(K)}
    \le \| f \|_{L^1(K)} \| \rho \|_{L^\infty(\RR)}.
\end{align*}
Therefore, put
\begin{align*}
    \phi_a(ru) := \min\left\{ \|\rho\|_{L^\infty(\RR)}\|f\|_{L^1(K)}, C_{\rho,s} \Phi_s[f](u) r^{\frac{-2s-m}{2}} \right\}, \:\:\:
    \phi_b(b) := \frac{|R[f](ru,b)|}{\phi_a(ru)}.
\end{align*}
Since the estimate $|R[f](ru,b)| \le \phi_a(ru)$ is independent of $b$, $\phi_b$ is well-defined and uniformly bounded as $|\phi_b| \le 1$.
By the square integrability of $R[f]$, we can decompose the integral as
\begin{align*}
    \| R[f] \|_{L^2(\RR^m\times\RR)}^2 = \int_{\RR^m} |\phi_a(a)|^2 \dd a \int_\RR |\phi_b(b)|^2 \dd b.
\end{align*}

%\subsection{Theorem \ref{thm:mainclaim}\ref{thm:mainclaim}}\label{supp:proof.mainclaim}
\subsection{\refthm{mainclaim}}\label{supp:proof.mainclaim}

\begin{proof}
We write $C_0 := \|\sigma\|_{L^\infty(\RR)} \|f\|_{L^1(K)}$ and $C_\infty := C_{\sigma,s}\|f\|_{H^s(\Omega)}$ for short. Let $V_a := \{ a \in \RR^m \mid |a| \le \lambda \}$ and $V_b := \{b \in \RR \mid |b| \le \kappa\}$ so that $V = V_a \times V_b$.
By \refthm{minimizer},
 the approximation error $\inf_{\mu \in \M (p)}\| f - S[\mu] \|_{L^2(K)}^2$ is lower bounded by the tail bound $\| S_K^*[f]\|_{L^2(V^c)}^2 = \| f \|_{L^2(K)}^2 - \| S^*[f] \|_{L^2(V)}^2$.
On the other hand, by \refthm{ubound},
the parameter ``density'' $|S_K^*[f]|^2$ is upper bounded by a dominating function $|\phi_a|^2$;
Furthermore, the integration of $|S_K^*[f]|^2$ over a product space $V_a \times V_b$ is exactly decomposed into the integrations of $|\phi_a|^2$ and $|\phi_b|^2$.
%by \min\{ \|f\|_{L^1(\RR^m)}^2 \| \sigma \|_{L^\infty(\RR)}^2, C_{\sigma,s}^2\|f\|_{H^s}^2r^{-2s-m}\}\min\{ \|f\|_{L^1(\RR^m)}^2 \| \sigma \|_{L^\infty(\RR)}^2, C_{\sigma,s}^2\|f\|_{H^s}^2r^{-2s-m}\}.
In the following, by integrating the dominating function over the bandlimited domain $V$,
we estimate the tail bound. % as $\| S_K^*[f] \|_{L^2(V^c)}^2 \ge \| f \|_{L^2(K)}^2 - \| \psi \|_{L^2(V)}^2$.

% For any interval B \subset \RRB \subset \RR, let
% \begin{align}
%     \Phi_b(B) := \int_B |\phi_b(b)|^2 \dd b.
% \end{align}
% By construction, \Phi_b\Phi_b is an absolutely continuous finite Borel measure on \RR\RR with volume
% \Phi_b(\RR) = \| f \|_{L^2(\RR^m)}^2 / \| \phi_a \|_{L^2(\RR^m)}^2\Phi_b(\RR) = \| f \|_{L^2(\RR^m)}^2 / \| \phi_a \|_{L^2(\RR^m)}^2,
% which will be refined later. Then,
We begin with decomposing the integral as
%Our integral in question is decomposed as
\begin{align*}
    \| S_K^*[f] \|_{L^2(V)}^2
    &\!= \!\left(\!\int_{\Sph^{m-1}} \!\int_0^\lambda |\phi_a(ru)|^2 r^{m-1} \dd r \dd u \!\right) \!\int_{V_b} \!\!|\phi_b(b)|^2 \dd b = \| \phi_a \|_{L^2(V_a)}^2 \| \phi_b \|_{L^2(V_b)}^2.
\end{align*}
Thus, we compute $\| \phi_a \|_{L^2(V_a)}^2$ in the following.
%Hence, we can concentrate on estimating \phi_a\phi_a.

By averaging $\phi_a$ in direction $u \in \Sph^{m-1}$,
\begin{align*}
%    \int_{\Sph^{m-1}} | S^*[f](ru,b) |^2 \dd u
    \int_{\Sph^{m-1}} | \phi_a(ru) |^2 \dd u
    &= \min \{ C_0^2 \Omega_{m-1}, C_\infty^2 r^{-2s-m} \}.
\end{align*}
Here, $\Omega_{m-1} = 2 \pi^{m-1}/\Gamma(m/2)$ is the surface area of $\Sph^{m-1}$.
Therefore, the rate in $r$ changes at the cross point $r = \vartheta$ satisfying $C_0^2 \Omega_{m-1} = C_\infty^2 \vartheta^{-2s-m}$.

Let us consider the case $\lambda \le \vartheta$. %where \int_{\Sph^{m-1}} \phi_a(ru) \dd u \equiv C_0\int_{\Sph^{m-1}} \phi_a(ru) \dd u \equiv C_0.
%Then, |S^*[f](a,b)| \le C_0|S^*[f](a,b)| \le C_0 for every (a,b) \in V(a,b) \in V.
Then,
\begin{align*}
    \| \phi_a \|_{L^2(V_a)}^2
%    &= \int_{\Sph^{m-1}} \int_0^\lambda \int_{-\kappa/2,\kappa/2} |S^*[f](ru,b)|^2 r^{m-1} \dd r \dd u \dd b \\
%    &\le C_0^2 V_m \kappa \lambda^m =: I_0(\lambda).
    &= \left(\int_{\Sph^{m-1}} \int_0^\lambda |\phi_a(ru)|^2 r^{m-1} \dd r \dd u \right)= C_0^2 V_m \lambda^m =: I_0(\lambda).
%    &\le C_0^2 V_m \kappa \lambda^m =: I_0(\lambda).
\end{align*}
Here, $V_m = \pi^{m/2}/\Gamma(m/2+1)$ is the volume of $m$-unit ball, and we used the relation $\Omega_{m-1}/m = V_m$.
Next, let us consider the case $\lambda \ge \vartheta$. %, where \phi_a(ru) = C_\infty \phi_a(ru) = C_\infty 
\begin{align*}
    \| \phi_a \|_{L^2(V_a)}^2
    &= I_0(\vartheta) + \int_{\Sph^{m-1}} \int_\vartheta^\lambda |\phi_a(ru)|^2 r^{m-1} \dd r \dd u \\
    &= C_0^2 V_m \vartheta^m - \frac{C_\infty^2}{2s}\left( \lambda^{-2s} - \vartheta^{-2s} \right) \\
    &= \frac{C_\infty^2}{2s} \left( -\lambda^{-2s} + \frac{2s + m}{m} \vartheta^{-2s}  \right),
\end{align*}
where the final equation is immediate from the relation $m C_0^2 V_m \vartheta^m = C_\infty^2 \vartheta^{-2s}$. By the positivity of integrand $|\phi_a|^2$, the final estimate is also positive (inspite of the negative term $-\lambda^{-2s}$).

As a biproduct, by letting $\lambda \to \infty$, we can verify that both $\phi_a$ and $\phi_b$ are finite measures on $\RR^m$ and $\RR$ respectively:
\begin{align*}
    &\| \phi_a \|_{L^2(\RR^m)}^2 = C_\infty^2 \left( \frac{2s+m}{2sm} \right) \vartheta^{-2s} \in (0,\infty), \\  &\implies \| \phi_b \|_{L^2(V_b)}^2 = \| S^*[f] \|_{L^2(\RR^m\times V_b)}^2 / \| \phi_a \|_{L^2(\RR^m)}^2 \in (0,\infty).
\end{align*}

To conclude, we have the following approximation lower bound:
\begin{align*}
    &\inf_{\mu \in \M(\num)} \| f - S[\mu] \|_{L^2(K)}^2 \\
    &\ge \inf_{\mu \in \M(V)} \| f - S[\mu] \|_{L^2(K)}^2 \\
    &\ge \| S_K^*[f]  \|_{L^2(V^c)}^2\\
    &= \| f \|_{L^2(K)}^2 - \| S^*[f] \|_{L^2(V)}^2 \\
    &= \|f\|_{L^2(K)}^2
    - \| \phi_b \|_{L^2(V_b)}^2 \cdot\begin{cases}
    \| f \|_{L^1(K)}^2 \| \sigma \|_{L^\infty(\RR)}^2 V_m \lambda^m & \lambda \in [0,\vartheta)\\
     \| f \|_{H^s(\Omega)}^2 C_{\sigma,s}^2 \left( -\frac{1}{2s}\lambda^{-2s}  + \frac{2s + m}{2sm} \vartheta^{-2s}  \right) & \lambda \in [\vartheta, \infty) \\
    \end{cases},
\end{align*}
where the final bound is continuous at $\lambda = \vartheta$, and it is non-negative.
\end{proof}

%%%%%%%%%%%%%%%%%%%%%%%%%%%%%%%%%%%%%%%%%%%%%%%%%%%%%%%%%%%%%%%%%%%%%%%%%%%%%%%
\section{Further Experiments}\label{supp:furtherexp}
\subsection{Simulations on a 2D artificial example }\label{supp:furtherexp-2D}
To further verify our results, we extend the 1D target function to a 2D case, which is expressed as follows:
\begin{equation*}
f_{2D}(x_1,x_2;\sigma) \!= \!0.2\exp\!\left(\!\!-\frac{(x_1-0.4)^{2}+(x_2-0.4)^{2}}{\sigma^2}\!\right)+0.5\exp\!\left(\!\!-\frac{(x_1-0.6)^{2}+(x_2-0.6)^{2}}{\sigma^2}\!\right),
\end{equation*}
where $x_1\in[0,1]$, $x_2\in[0,1]$,${\sigma}>0$ is a scalar index that can determine the complexity of $f_{2D}$, similar as the 1D case.

Similar to Simulation 2 (1D case) as detailed in Section \ref{sec4}, we create different forms of target function $f(x_1,x_2;\sigma)$ by choosing $\sigma$ as one element of the set $\{0.01,0.05,0.1,0.5\}$, and for each regression task we build random nets with $\lambda$ taken as an element from the set $\{0.1,0.5,1,5,10,50,100,200\}$, and fix the number of hidden nodes as $L=10000$ for each case. We sample 10000 instances $\{(x_{1}^{(i)},x_{2}^{(i)}),f_{2D}(x_{1}^{(i)},x_{2}^{(i)})\}_{i=1}^{10000}$ which are meshgrid points on $[0,1]^2$ (both $x_1$ and $x_2$ are equally space points over [0,1]), then randomly and uniformly select 5000 training samples and 5000 test samples. 

For each pair ($\lambda,\sigma$), we run independently 50 trials and calculate the relative training error for each trial. The following Table \ref{table:2D_case} shows the averaged training performance for the case of each pair ($\lambda,\sigma$). 

\begin{table}[htbp!]
\caption{Summary of mean relative training error for various choices of ($\lambda,\sigma$) for the 2D case.}
\label{table:2D_case}
\vskip 0.1in
\begin{center}
\begin{tabular}{lcccc}
\hline
\multirow{2}{*}{$\lambda$}      & \multicolumn{4}{c}{Averaged Relative Training Error $E$} \\
\cline{2-5} & $\sigma=0.01$                     & $\sigma=0.05$                    & $\sigma=0.1$         & $\sigma=0.5$              \\
 \hline
$\lambda=0.1$                            & 0.0310                               &  0.0225                               &  0.0121          &
0.0062
\\
$\lambda=0.5$                            & 0.0297                               &  0.0214                               &  0.0086          &
0.0041
\\
$\lambda=1$                              & 0.0296                               &  0.0210                               &  0.0072          &
0.0016
\\
$\lambda=5$                              & 0.0277                               &  0.0032                               &  0.0012          &
2.8661e-04
\\
$\lambda=10$                             & 0.0192                               &  0.0011                               &  3.4762e-05         &   2.1093e-05                      \\
$\lambda=50$                             & 0.0010                               & 6.3672e-05        &  6.1358e-05         &  5.3784e-05                            \\
$\lambda=100$                            & 1.2561e-04       & 4.2462e-05        & 5.1378e-05           & 5.3165e-05                           \\
$\lambda=200$                        & 1.1762e-04       & 3.2672e-05        & 2.6826e-05           & 2.3018e-05                             \\
\hline
\end{tabular}
\end{center}
\vskip -0.1in
\end{table}

It is clear that similar findings can be seen from Table \ref{table:2D_case}, that is, consistent with the conclusion drwan from Table \ref{table:1}, there exists an appropriate range/distribution $\mathcal{D}^{*}$, but \textbf{NOT ANY} range/distribution, such that a neural network with random weights (NNRWs) assigned from $\mathcal{D}^{*}$ can be a universal approximator. Essentially, the $\mathcal{D}^{*}$ (e.g, $[-\lambda^*,\lambda^*]^2$) is highly dependent upon the complexity of the target function, as consistent with the theoretical and empirical results elaborated in \cite{li2017insights}.

\subsection{Simulations on five real-world datasets}\label{supp:furtherexp-real}
Also, we conduct another simulation study on five real-world datasets from KEEL-dataset repository for regression task (\url{https://sci2s.ugr.es/keel/}). The basic information of these datasets is summarized in Table \ref{table:realdatasets}. We choose randomly 75$\%$ samples as traning set while the left samples for testing set. Similar as the experiments conducted on 1D and 2D artificial examples presented before, we also consider different settings of $\lambda$ for each dataset, and fix $L=10000$ for the neural network with random weights. Then, we run independently 50 trials and calculate the relative training error for each trial. The following Table \ref{table:Performance_realdatasets} shows the averaged training performance for the case of each pair ($\lambda,\sigma$). 

\begin{table}[htbp!]
\caption{Summary of basic information of five real-world datasets}\label{table:realdatasets}
\begin{center}
\begin{tabular}{cccc}
\hline
Dataset  & Number of Samples & Input Dimension &Output Dimension\\
\hline
stock    & 950               & 9                &1 \\
laser    & 993               & 4                &1  \\
friedman & 1200              & 5                 &1 \\
abalone  & 4177              & 8   &1 \\
compactiv & 8192  & 21 &1\\
\hline
\end{tabular}
\end{center}
\end{table}

\begin{table}[htbp!]
\caption{Summary of mean relative training error for various choices of ($\lambda,\sigma$) for real-world datasets.}
\label{table:Performance_realdatasets}
\vskip 0.1in
\begin{center}
\begin{tabular}{cccccc}
\hline
\multirow{2}{*}{$\lambda$}      & \multicolumn{4}{c}{Averaged Relative Training Error $E$} \\
\cline{2-6} & stock                     & laser                    & friedman        & abalone      &   compactiv     \\
 \hline
$\lambda=0.1$                            & 0.0065                               &  0.0131                              &  0.0314          &
0.0654 &0.0145
\\
$\lambda=0.5$                            & 1.4295e-09                               &  0.0111                               &  0.0027          &
0.0468&0.0029
\\
$\lambda=1$                              & 1.0003e-10                               &  0.0103                               &  1.1831e-09          &
0.0300&9.8862e-04
\\
$\lambda=5$                              & 2.2323e-13                               &  5.8883e-04                               &  1.8764e-13          &
3.5419e-09&3.0985e-08
\\
$\lambda=10$                             & 2.6994e-14                               &  4.2153e-10                               &  9.1391e-15         &   9.7869e-11              &5.8270e-09        \\
$\lambda=50$                             & 3.0819e-15                               & 4.4011e-14        & 3.1680e-15        &  1.5955e-13                           &2.3818e-10 \\
$\lambda=100$                            & 2.8372e-15       & 8.9380e-15        & 2.9847e-15           & 6.2920e-14                           &5.3666e-11 \\
$\lambda=200$                        & 5.2748e-15       & 2.9870e-15        & 3.1563e-15           & 1.6981e-14                            &3.0773e-11 \\
\hline
\end{tabular}
\end{center}
\vskip -0.1in
\end{table}
It clearly shows that there are a few cases (like $\lambda=0.1,0.5,1$) when the training errors of the randomized neural networks cannot converge to zero (even when $L=10000$). This finding is also consist with what we have obtained in the 1D and 2D artificial examples. All these findings validate our theoretical results that when hidden parameters are distributed in
a bounded domain, the network may not achieve zero approximation error.

\subsection{Quantitative demonstration for Figure 1}\label{Quantitative}
Consider the following two toy examples:
\begin{equation*}
f_1(x) = 0.2e^{-(10x-4)^{2}}+0.5e^{-(80x-40)^{2}}+0.3e^{-(80x-20)^{2}}, \:x\in[0,1],
\end{equation*}
and
\begin{equation*}
f_2(x) = 0.8\exp(-0.2x)\sin(10x), \:x\in[0,5].
\end{equation*}
We uniformly sample 1000 training samples (with $x\in [0,1]$ for $f_1$, $x\in [0,5]$ for $f_2$, respectively). For the neural networks with random weights (NNRWs), we fix the number of hidden nodes $L=10000$ (so that we can observe the trend as $L \rightarrow\infty$) and try the randomized learner model using different setting of the random distribution, e.g., $\lambda=[0.1,0.2,0.3,0.4,0.5,0.6,0.7,0.8,0.9,1,5,10,15,20,25,30,40,50]$. As shown clearly in \reffig{appendix_figure} (similar to the qualitative plot shown in \reffig{lbound}), for both toy examples, when $\lambda=[0.1,0.2,0.3,0.4,0.5,0.6,0.7,0.8,0.9,1]$ the approximation error change is relatively flatten, while for $\lambda=[5,10,15,20,25,30,40,50]$ the magnitude of the approximation error decreasing become much larger. Although it is intuitively seen that the threshold value $\vartheta$ is `roughly' around 1 for both $f_1$ and $f_2$, it is not easy to find the `optimal' value of $\vartheta$. Given limited training samples (sampled from an unknown function), how to develop advanced algorithms/strategies to compute numerically the threshold $\vartheta$ is out of the focus of our current work. Nonetheless, it is expected to benefit and motivate future research on algorithm development for building more powerful (shallow and/or deep) neural nets with random weights (NNRWs).

\begin{figure}[htbp!]
%\vspace{-0.5cm}
\centering
\subfigure[$f_1$]{\includegraphics[width=0.45\textwidth]{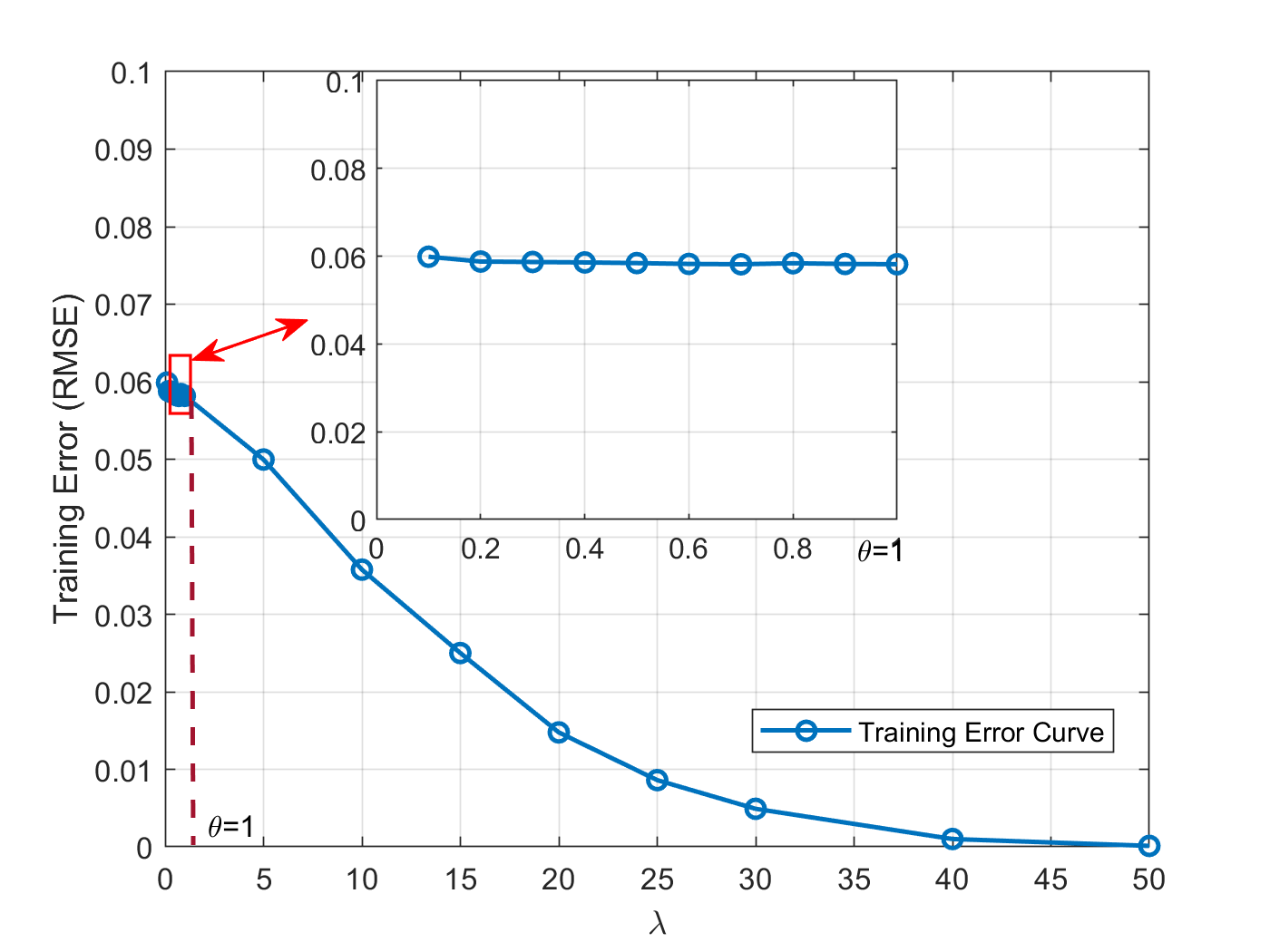}}
\subfigure[$f_2$]{\includegraphics[width=0.45\textwidth]{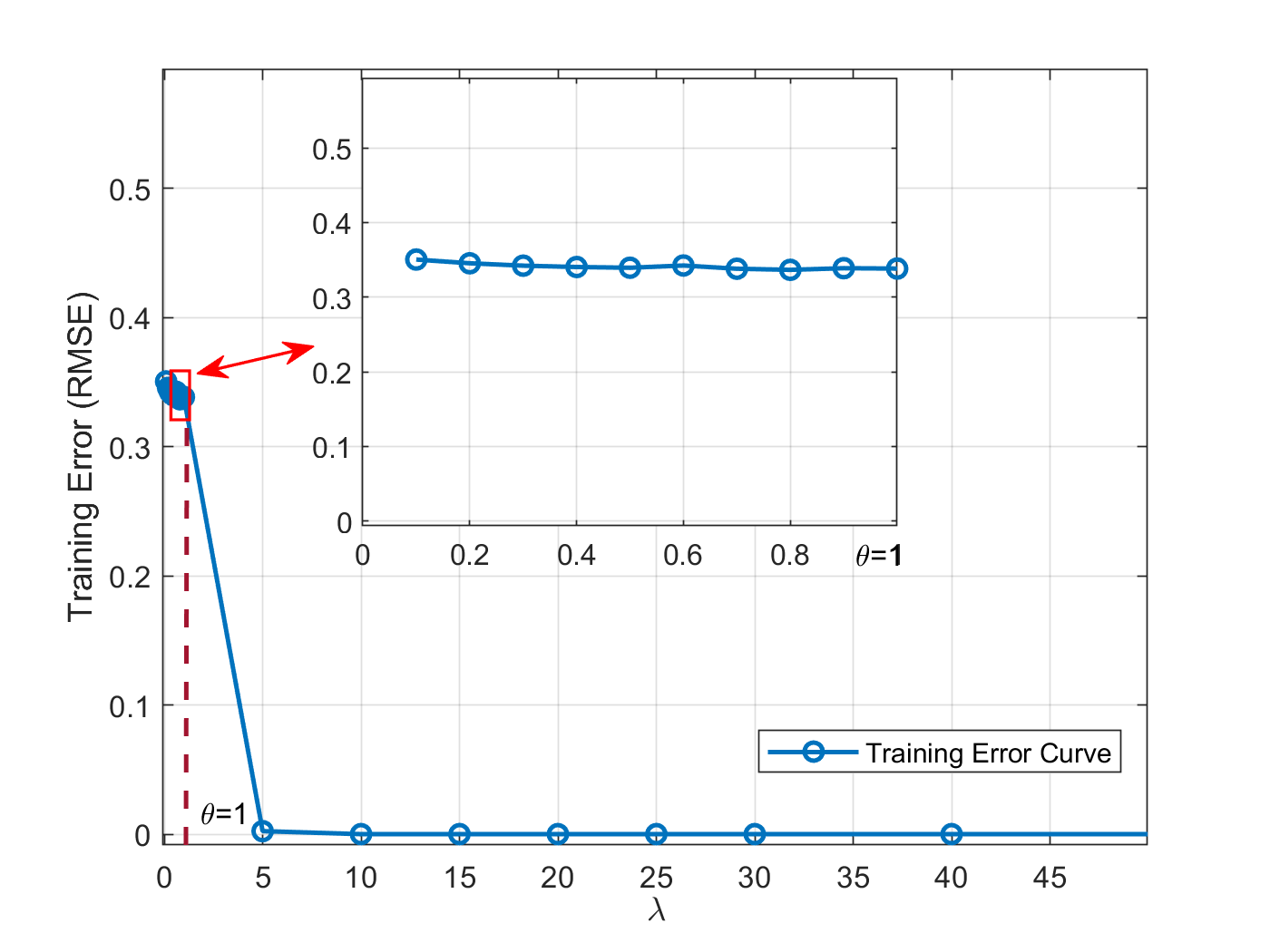}}
\caption{Outline of the approximation lower bound for $f_1$ and $f_2$.}\label{fig:appendix_figure}
\end{figure}
\end{document}